\documentclass[sigconf]{acmart}

\usepackage{microtype}
\usepackage{enumitem}
\usepackage{bbm}
\usepackage{subcaption}
\usepackage{graphicx}
\usepackage{booktabs}
\usepackage{cancel}
\usepackage{hyperref}
\usepackage{xcolor}
\usepackage{balance}

\newcommand\Ccancel[2][black]{\renewcommand\CancelColor{\color{#1}}\xcancel{#2}}
\newcommand\independent{\protect\mathpalette{\protect\independenT}{\perp}}
\def\independenT#1#2{\mathrel{\rlap{$#1#2$}\mkern2mu{#1#2}}}
\usepackage{xcolor,soul}
\sethlcolor{lightgray}
\usepackage{graphicx}
\usepackage{amsthm}
\usepackage{amsmath}
\usepackage{tikz}
\usetikzlibrary{arrows.meta}
\newtheorem{assumption}{Assumption}

\AtBeginDocument{%
  }

\copyrightyear{2026}
\acmYear{2026}
\setcopyright{cc}
\setcctype{by}
\acmConference[SIGIR '26]{Proceedings of the 49th International ACM SIGIR Conference on Research and Development in Information Retrieval}{July 20--24, 2026}{Melbourne, VIC, Australia}
\acmBooktitle{Proceedings of the 49th International ACM SIGIR Conference on Research and Development in Information Retrieval (SIGIR '26), July 20--24, 2026, Melbourne, VIC, Australia}
\acmDOI{10.1145/3805712.3809673}
\acmISBN{979-8-4007-2599-9/2026/07}

\begin{document}

\title{Inductive Subgraphs as Shortcuts: Causal Disentanglement for Heterophilic Graph Learning}

\author{Xiangmeng Wang}
\authornote{Both authors contributed equally to this research.}
\email{xiangmengpoly.wang@polyu.edu.hk}
\orcid{0000-0003-3643-3353}
\affiliation{
  \institution{The Hong Kong Polytechnic University}
  \city{Kowloon}
  \country{Hong Kong}
}

\author{Qian Li}
\authornotemark[1]
\email{qli@curtin.edu.au}
\orcid{0000-0002-8308-9551}
\affiliation{
  \institution{Curtin University}
  \city{Perth}
  \country{Australia}
}

\author{Haiyang Xia}
\email{hyxia@tulip.academy}
\orcid{0000-0002-0363-1460}
\affiliation{
  \institution{University of Macau}
  \city{Taipa}
  \country{Macau}
}

\author{Hao Miao}
\email{hao.miao@polyu.edu.hk}
\orcid{0000-0001-9346-7133}
\affiliation{
  \institution{The Hong Kong Polytechnic University}
  \city{Kowloon}
  \country{Hong Kong}
}

\author{Qing Li}
\authornote{Corresponding author.}
\email{qing-prof.li@polyu.edu.hk}
\orcid{0000-0003-3370-471X}
\affiliation{
  \institution{The Hong Kong Polytechnic University}
  \city{Kowloon}
  \country{Hong Kong}
}

\author{Guandong Xu}
\email{gdxu@eduhk.hk}
\orcid{0000-0003-4493-6663}
\affiliation{
  \institution{The Education University of Hong Kong}
  \city{New Territories}
  \country{Hong Kong}
}

\renewcommand{\shortauthors}{Xiangmeng Wang et al.}

\begin{abstract}
Heterophily is a prevalent property of real-world graphs and is well known to impair the performance of homophilic Graph Neural Networks (GNNs). Prior work has attempted to adapt GNNs to heterophilic graphs through non-local neighbor extension or architecture refinement. However, the fundamental reasons behind misclassifications remain poorly understood. In this work, we take a novel perspective by examining recurring inductive subgraphs, empirically and theoretically showing that they act as spurious shortcuts that mislead GNNs and reinforce non-causal correlations in heterophilic graphs. To address this, we adopt a causal inference perspective to analyze and correct the biased learning behavior induced by shortcut inductive subgraphs. We propose a debiased causal graph that explicitly blocks confounding and spillover paths responsible for these shortcuts. Guided by this causal graph, we introduce Causal Disentangled GNN (CD-GNN), a principled framework that disentangles spurious inductive subgraphs from true causal subgraphs by explicitly blocking non-causal paths. By focusing on genuine causal signals, CD-GNN substantially improves the robustness and accuracy of node classification in heterophilic graphs. Extensive experiments on real-world datasets not only validate our theoretical findings but also demonstrate that our proposed CD-GNN outperforms state-of-the-art heterophily-aware baselines.
\end{abstract}

\ccsdesc[500]{Mathematics of computing~Graph algorithms}
\ccsdesc[500]{Computing methodologies~Neural networks}

\keywords{Heterophilic Graphs; Graph Neural Networks; Causal Inference}


\maketitle

\section{Introduction}
\label{sec:introduction}
Heterophilic graphs are common in practice, where linked nodes often have different features and labels~\cite{zhu2024impact,gao2023addressing}. For example, in telecom fraud detection~\cite{jiang2022telecom, ren2024heterophilic}, scammers typically contact many legitimate users rather than other scammers, forming a heterophilous interaction graph. Such a heterophily property conflicts with the homophily assumption used in many message-passing GNNs (e.g., GCN~\cite{kipf2016semi}, GIN~\cite{xu2018powerful}) and has been shown to significantly degrade their node classification performance~\cite{luan2024heterophilic,bo2021beyond,pei2020geom,lei2025divergent}.

Prior work has mainly improved GNN performance on heterophilic graphs through two directions: non-local neighbor extension and architecture refinement. Non-local extension enlarges the receptive field to include more distant nodes, based on the intuition that similar nodes may be far apart under heterophily. Representative methods include higher-order neighbor mixing~\cite{abu2019mixhop} and potential neighbor discovery~\cite{pei2020geom}. Architecture refinement instead redesigns message passing, using spectral-based filters~\cite{zheng2023node,he2021bernnet,guo2023graph,bo2021beyond} or spatial-based mechanisms~\cite{yang2021diverse, chen2022graph,yan2022two}.
Spectral and spatial methods tackle heterophily by separating similarity from dissimilarity signals or by learning adaptive neighbor aggregation. 
In the spectral domain, FAGCN~\cite{bo2021beyond} employs a self-gating mechanism to capture both low-frequency and high-frequency components, where high-frequency signals encode neighborhood dissimilarity that standard low-pass GNNs often overlook.
ACM~\cite{luan2022revisiting} further combines low-pass, high-pass, and identity filtering within a multi-channel GCN. In the spatial domain, DMP~\cite{yang2021diverse} learns edge-specific propagation weights via attention to handle attribute heterophily, while GDAMNs~\cite{chen2022graph} models edge uncertainty with variational inference and learns attention-based edge weights to handle label heterophily.

\begin{figure}
    \centering
    \includegraphics[width=\linewidth]{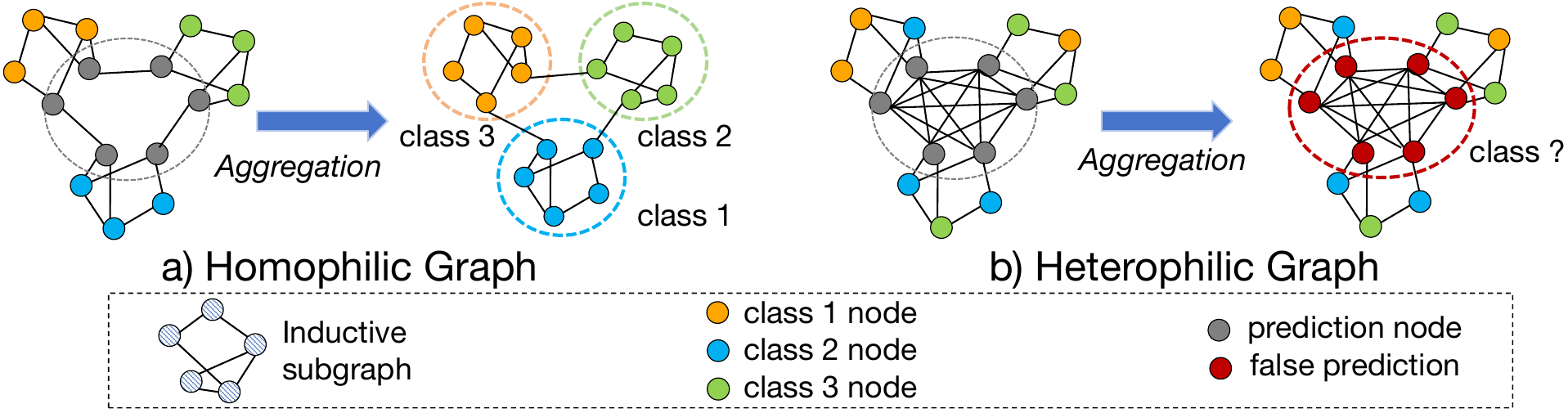}
    \caption{A toy example shows that an inductive subgraph helps a homophilic graph infer correct labels, while harming a heterophilic graph for false label prediction.}
    \label{fig:toy}
\end{figure}

Despite substantial progress on learning from heterophilic graphs, existing methods largely overlook inductive subgraphs as a hidden source of bias. 
Recent studies suggest that inductive subgraphs are ubiquitous in real-world graphs, manifesting as recurring local patterns where nodes and their neighborhoods exhibit consistent feature or label regularities~\cite{ying2019gnnexplainer,luo2020parameterized}. As illustrated in Figure~\ref{fig:toy}, such patterns can appear as repeated structural motifs (\raisebox{-0.25\height}{\includegraphics[height=1em]{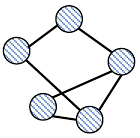}}) that are subgraphs encoding local structural and label regularities. While these structural priors are often viewed as helpful for inductive learning on homophilic graphs ({Figure~\ref{fig:toy} (a)}); in heterophilic graphs, they can induce inductive biases~\cite{battaglia2018relational} that reinforce incorrect associations between structual patterns and node labels ({Figure~\ref{fig:toy} (b)}), leading to degraded performance~\cite{ma2021homophily,luan2024heterophilic}. However, existing work has largely failed to systematically examine the role of structural priors in heterophilic settings.
As a result, the interaction between inductive graph patterns and heterophily remains poorly understood.

To fill in this gap, we improve heterophilic graph learning by inspecting and correcting structural inductive patterns from a graph-centric perspective.
We conduct comprehensive controlled experiments (Section~\ref{sec:empirical}) and develop rigorous theoretical results (Section~\ref{sec:inductive-subgraph-movements}). Our findings consistently show that inductive subgraphs are benign in homophilic graphs but can be malignant in heterophilic graphs, where they act as shortcuts that are easier for GNNs to fit while encoding causally spurious associations. We further interpret heterophilic GNN learning through a causal lens to explain when and why these shortcut effects arise.
Our causal analysis identifies two unblocked causal paths, a confounding path and a spillover path, that induce misleading associations between inductive subgraphs and model predictions. These spurious associations degrade standard message passing GNNs, especially in heterophilic settings where linked nodes often have different labels. Motivated by this insight, we propose the \textit{Causal Disentangled GNN (CD-GNN)}, which explicitly disentangles spurious inductive subgraphs from true causal factors without modifying the underlying GNN architecture. CD-GNN blocks both confounding and spillover paths in the causal graph so that predictions depend only on genuine causal influences. The resulting disentangled factors are represented as subgraphs that reflect the true drivers of the predictions, leading to more robust performance consistent with observed outcomes. To the best of our knowledge, this is the first work to study heterophilic graph learning from the perspective of inductive subgraphs.~\footnote{Recent works~\cite{yan2022two,liu2023beyond} attribute heterophily to node-level oversmoothing. In contrast, we identify structural inductive subgraphs as an overlooked source of degradation and shift the analysis from node-level dynamics to a structural perspective.}
Our major contributions are as follows:
\begin{itemize}
    \item \textbf{New Insights:} We provide empirical and theoretical analysis showing that inductive subgraphs can create spurious shortcuts that degrade GNN performance on heterophilic graphs. From a causal inference perspective, we identify two unblocked causal paths, confounding and spillover, that induce misleading associations between subgraphs and predictions. We further provide theoretical guarantees that our method mitigates this degradation in heterophilic settings.
    \item  \textbf{Causality‐driven Debiasing}: We propose CD-GNN, which de-biases the causal graph by explicitly blocking the confounding and spillover paths. It disentangles spurious inductive subgraphs from true causal factors, ensuring predictions rely on genuine causal influences.
\item  \textbf{Causality-Aligned Interpretation}: CD-GNN provides robust predictions in challenging heterophilic settings and provides subgraph-level interpretations that reflect the underlying causal drivers of its decisions.
\end{itemize}

\section{Notation and Preliminaries}
Let $\mathcal{G}=(\mathcal{V},\mathcal{E})$ be an undirected, unweighted graph with node set $\mathcal{V}$ and edge set $\mathcal{E}$. The neighborhood of node $v$ is denoted by $N(v)$, which may include $v$ due to self-loops. Graph connectivity is represented by an adjacency matrix $\mathbf{A}\in\{0,1\}^{|\mathcal{V}|\times|\mathcal{V}|}$. Each node $v$ has a $d$-dimensional feature vector $\mathbf{x}_v\in\mathbb{R}^d$, and all node features form the matrix $\mathbf{X}\in\mathbb{R}^{|\mathcal{V}|\times d}$.
We consider the node classification task, where each node $v$ has a label $y_v\in\mathcal{Y}$ from $C$ classes. 

\noindent \textbf{Heterophilic Graphs.}
Graph heterophily occurs when connected nodes have dissimilar features and labels, in contrast to homophily where similar nodes are more likely to connect.
Inspired by~\cite{luan2024heterophilic}, we measure heterophily using two complementary metrics. Specifically, we use the label heterophily ratio $h_L$\cite{pei2020geom} and the feature heterophily ratio $h_F$\cite{jin2022raw}.
$h_L$ is the fraction of edges connecting nodes with different labels, and $h_F$ measures feature dissimilarity via the cosine function.
Strongly heterophilic graphs have $h_L, h_F \rightarrow 1$; strongly homophilic graphs have $h_L, h_F \rightarrow 0$.

\begin{definition}[Inductive Subgraph Bias] 
\label{def:inductive graph}
A graph with $M < |\mathcal{V}|$ nodes is an inductive subgraph of $\mathcal{G}$ if the two conditions hold simultaneously:
(a) The ego node and its neighbors form a particular graph structure and share similar node features, and (b) nodes within the same graph structure share the same labels or follow a pre-defined (e.g., nodes at the top share the same label) labeling pattern.
Inductive subgraph induces \textit{inductive bias}~\cite{battaglia2018relational}, where consistent relationships between graph structure and node labels allow the model to predict unseen instances based on learned patterns.
\end{definition}

\section{Empirical Study and Theoretical Proof}
This work investigates a fundamental yet underexplored question:
\begin{itemize}
    \item \textit{Is the inductive subgraph the root cause of erroneous or biased outcomes when training GNNs on heterophilic graphs}?
\end{itemize}

\subsection{Empirical Study}
\label{sec:empirical}
\textbf{Setup.}
To study the effect of inductive subgraphs on GNNs in homophilic and heterophilic settings, we conduct an empirical study on \texttt{Tree-Cycles}, \texttt{Tree-Grid}, \texttt{BA-Shapes}, and \texttt{BA-Community} in~\cite{ying2019gnnexplainer}.
As shown in Figure~\ref{fig:empirical explanation}, all four datasets comprise multiple base and motif subgraphs. 
Within any single subgraph, all nodes share the same feature (i.e., a uniform feature $[1,1,...1]$). 
Labels are assigned according to 1) Uniform labeling: every node in the base or motif subgraphs has the same class label. 2) Role-based labeling: a node's label depends on its position in the subgraph, e.g., top-level nodes in the house motif in \texttt{BA-Shapes} are class 1.
These base and motif subgraphs constitute inductive subgraphs (Definition~\ref{def:inductive graph}), providing strong inductive bias for GNN predictions.

We compute graph heterophily for each dataset using the label heterophily metric $h_L$.~\footnote{To evaluate node classification performance under heterophily, we iteratively relabel nodes to assign labels dissimilar to their neighbors, progressively increasing the graph's heterophily ratio until convergence, i.e., no further increase is observed over 50 iterations. 50 is a commonly adopted threshold in early stopping strategies~\cite{zhang2005boosting}.}
We compare three GNN models: GCN~\cite{kipf2016semi}, GAT~\cite{velivckovic2017graph}, and GraphSAGE~\cite{hamilton2017inductive} (Table~\ref{tab:empirical}). 
After training all three models, we chose GCN for analysis due to its popularity~\cite{xu2018powerful} to explain how inductive subgraphs influence prediction mechanisms.
We fix GCN's parameters and apply PGExplainer~\cite{luo2020parameterized} to explain node label predictions; the explanation results are given in Figure~\ref{fig:empirical explanation}.

\begin{table}[!h]
\caption{Comparison of GNN performance across homophilic and heterophilic settings.
The node classification accuracy is averaged over 5 independent training rounds. 
Avg Drop.\% denotes the average performance drop from the worst (bold) and best (underlined) performances across the three models. 
}\label{tab:empirical}
\centering
\setlength{\tabcolsep}{5pt}
\resizebox{0.48\textwidth}{!}{
\begin{tabular}
    {c ccc ccc ccc ccc} \toprule
    Dataset
    & \multicolumn{3}{c}{\texttt{Tree-Cycles}} 
    & \multicolumn{3}{c}{{\texttt{Tree-Grid}}} 
    & \multicolumn{3}{c}{ {\texttt{BA-Shapes}}}
    & \multicolumn{3}{c}{\texttt{BA-Community}}\\\cmidrule(lr){2-4} \cmidrule(lr){5-7} \cmidrule(lr){8-10} \cmidrule(lr){11-13}
$h_L$ & 0.098
 & 0.250 & 0.500  & 0.055
 & 0.250 & 0.500  & 0.200
 & 0.500 & 0.750
  & 0.264
 & 0.500
 & 0.850
 \\\midrule 
\midrule

GCN &\underline{0.988} &0.886 &\textbf{0.545}  &\underline{0.991} &0.935 &\textbf{0.532}  & \underline{0.957}  & 0.890  & \textbf{0.614}   &\underline{0.842} & 0.714 	&\textbf{0.501} \\
GAT &\underline{0.994} & 0.897&\textbf{0.591} &\underline{0.882}&0.741&\textbf{0.613}&\underline{0.971}&0.914&\textbf{0.584}&\underline{0.839}&0.735 & \textbf{0.464}\\
GraphSAGE &\underline{0.960}&0.802&\textbf{0.569}&\underline{0.821}&0.732&\textbf{0.580}&\underline{0.993}&0.871&\textbf{0.602}&\underline{0.801}&0.683&\textbf{0.512}\\
\midrule
\textbf{Avg Drop.\%} &  &-42.04\%  &  &  &-35.39\%  &  &  &-38.36\%  &  &  &-40.46\%  &\\

\bottomrule
\end{tabular}
}
\end{table}
\begin{figure}[ht]
    \centering
    \includegraphics[width=0.5\textwidth]{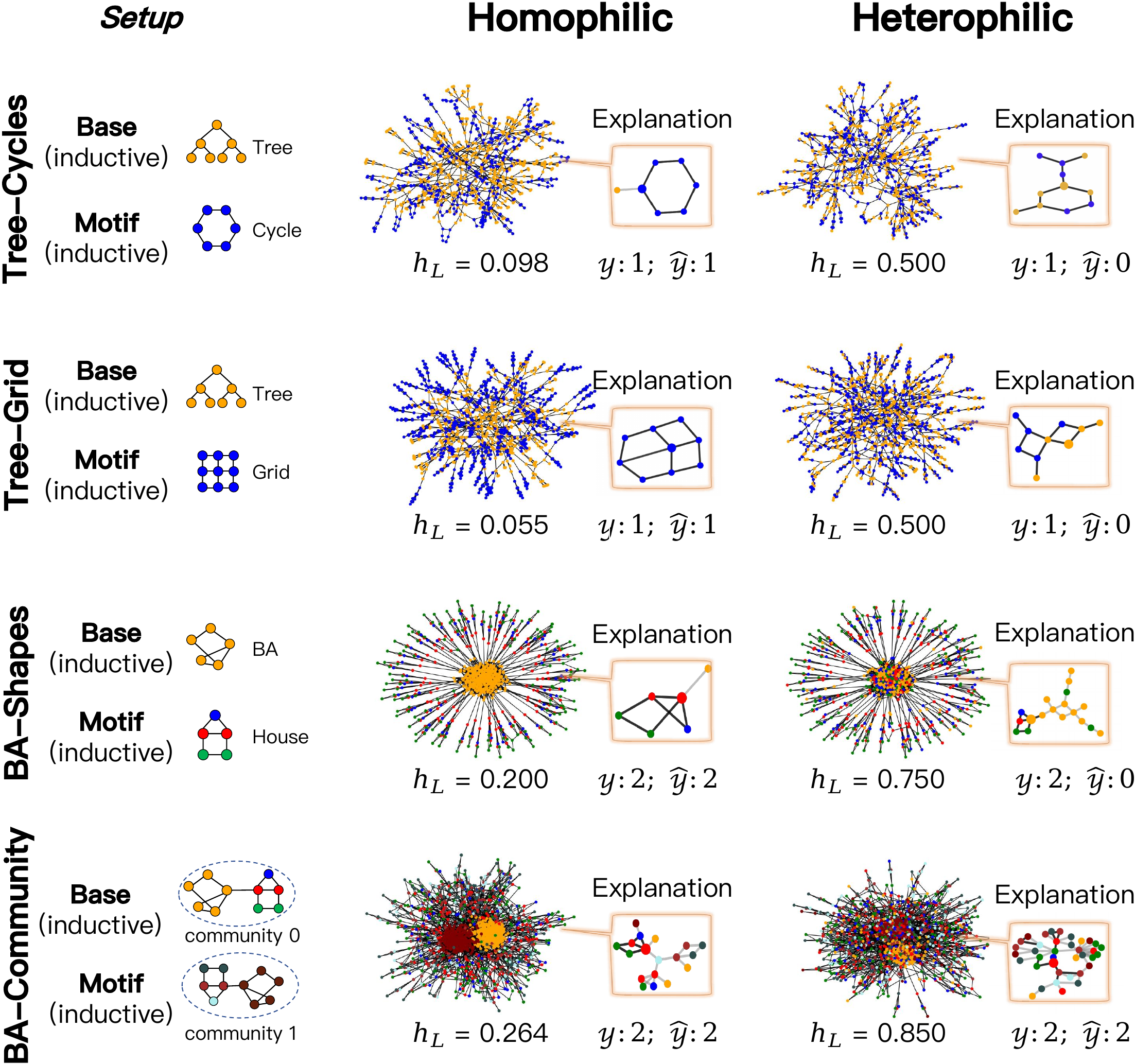}
    \caption{Prediction explanations: $y$, $\hat{y}$ is the ground truth and predicted label for the explaining node instance, respectively. 
    The node being explained is highlighted with a larger node size for clarity. 
    For consistency, the same node is selected for explanation across both homophilic and heterophilic settings within a single dataset.
    }
    \label{fig:empirical explanation}
\end{figure}

\begin{figure}[hbt!]
    \centering
    \includegraphics[width=0.5\textwidth]{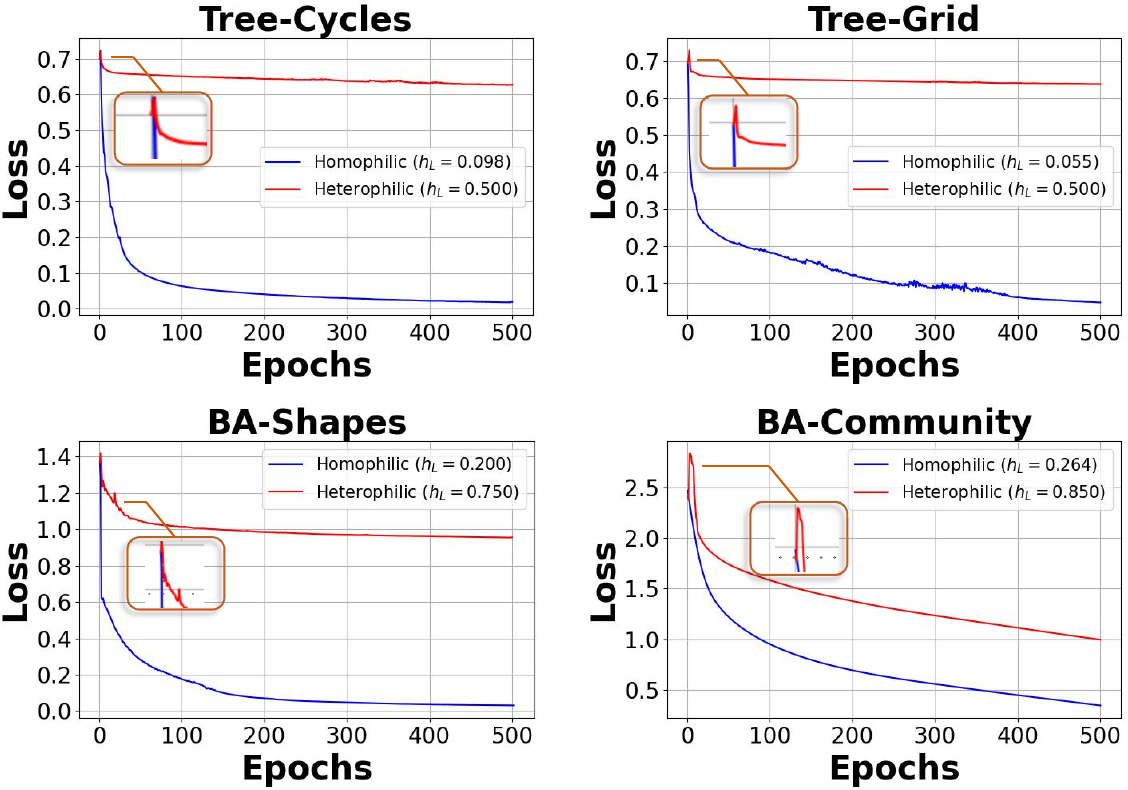}
    \caption{Training losses of the GNN model on homophilic and heterophilic settings.
    }
    \label{fig:empirical loss}
\end{figure}

\noindent \textbf{\textit{Observation 1}. Benign and malignant inductive subgraph.}
Table~\ref{tab:empirical} shows that GNN performance drops sharply as heterophily increases, showing their limited robustness on heterophilic graphs. For instance, GAT achieves only 0.464 accuracy on \texttt{BA-Community} with $h_L = 0.850$. To investigate this decline, we scrutinize GCN explanations in Figure~\ref{fig:empirical explanation}.
We observe a) \textbf{Benign Case}: On homophilic graphs (low heterophily), GCNs effectively exploit inductive subgraphs for label prediction. 
For example, in \texttt{Tree-Cycles}, a node with ground truth label $y=1$ exhibits a clear ``cycle'' (\raisebox{-0.25\height}{\includegraphics[height=1em]{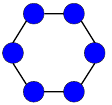}}) subgraph in its explanation, which the model uses as an inductive cue to predict $\hat{y}=1$. 
Similarly, distinctive ``grid'' (\raisebox{-0.25\height}{\includegraphics[height=1em]{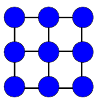}}) and ``house'' (\raisebox{-0.25\height}{\includegraphics[height=1em]{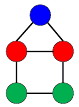}}) motifs are observed for correct predictions in \texttt{Tree-Grid} and \texttt{BA-Shapes}, respectively.
 b) \textbf{Malignant Case}: In heterophilic graphs, inductive subgraphs instead mislead GNNs and cause incorrect predictions. 
 For example, in \texttt{Tree-Cycles}, the GCN still identifies the ``cycle'' subgraph but predicts $\hat{y}=0$ since most neighbors within the ``cycle'' subgraph are labeled as $0$ (marked \tikz\fill[orange] (0,0) circle (3pt);). 
 This demonstrates how the same inductive structure introduces a malignant bias under heterophily, leading to misclassification.

\noindent \textbf{\textit{Observation 2}. Inductive subgraphs as shortcuts.}
Figure~\ref{fig:empirical loss} reveals distinct training dynamics of the GCN model under homophily and heterophily. In a homophilic graph, the model quickly learns benign inductive subgraphs aligned with the labels, leading to rapid loss reduction. In contrast, under heterophily, the model captures malignant inductive subgraphs, causing an increased training loss in the initial stage.
Therefore, inductive subgraphs act as misleading shortcuts in heterophilic graph learning: they are easy to capture during initial training, but do not reflect the true factors driving predictions (reflected by increased loss).
This behavior is consistent with \textit{shortcut learning}~\cite{geirhos2020shortcut,brown2023detecting}, where models favor simple, recurring patterns over more complex but causal structures.
We thus conclude that inductive subgraphs serve as shortcuts (denoted as $S$ throughout the paper) under heterophily, as they are easier to learn but are not causally responsible for correct predictions.

\subsection{Theoretical Proof}
\label{sec:inductive-subgraph-movements}
Prior work~\cite{yan2022two} explains degraded performance on heterophilic graphs via oversmoothing, where deeper layers make node representations increasingly similar despite dissimilar neighbors. This view, however, emphasizes representation dynamics and overlooks structural effects from specific subgraphs. In this section, we prove that inductive subgraphs can directly harm GNN performance on heterophilic graphs.
We adopt the renormalized linear propagation and the layer-wise effective homophily and relative degree notation from~\cite{yan2022two}. Thm~\ref{thm:theorem31prime} and Thm~\ref{thm:theorem32prime} provide a subgraph-aware characterization of how message passing alters class-related signals in heterophilic graphs, explaining why recurring motifs can degrade prediction accuracy.

\begin{theorem}[one-layer subgraph-aware conditional expectation]
\label{thm:theorem31prime}
Under Assumption and definitions in Section~\ref{subsec:assumptions_subgraph}, for any node $i$ at layer $\ell=0$ and conditioning on $(d_i,y_i,S_i^{(0)})$, the one-layer conditional expectation satisfies
\begin{equation}
\mathbb{E}\!\left[f_i^{(1)} \mid d_i,y_i,S_i^{(0)}\right]
=
\underbrace{\frac{1+W_i^0\Big((1+\rho)h^0_{i,\mathrm{eff}}-\rho\Big)}{d_i+1}}_{=:~G_i^{(0)}}
\;
\mathbb{E}\!\left[f_i^{(0)} \mid y_i\right]
\label{eq:C4_gain_simple}
\end{equation}
\end{theorem}
\label{subsubsec:theorem32prime}

\begin{theorem}[deeper-layer subgraph-aware conditional expectation]
\label{thm:theorem32prime}
We assume the baseline conditions under Theorem 3.2~\cite{yan2022two} holds.
Using the deep-layer dominance ratio and effective homophily in Assumption~\ref{ass:group_exchangeability} in Section~\ref{subsec:assumptions_subgraph},
we obtain the conditional expectation at layer $\ell+1$ for the \emph{subgraph-separated} form
\begin{equation}
\mathbb{E}\!\left[f_i^{(\ell+1)} \mid d_i,y_i,S_i^{(\ell)}\right]
=
\underbrace{\frac{\Big((1+\rho_i^\ell)\hat h^\ell_{i,\mathrm{eff}}-\rho_i^\ell\Big)\, d_i \bar r_i^\ell + 1}{d_i+1}}_{=:~G_i^{(\ell)}}
\;\xi_i^\ell\;
\mathbb{E}\!\left[f_i^{(0)} \mid y_i\right]
\label{eq:C5_gain_simple}
\end{equation}
\end{theorem}

Thm~\ref{thm:theorem31prime} shows: conditioning on $(d_i,y_i,S_i^{(0)})$, one propagation step scales class-conditional base signal by a
scalar gain $G_i^{(0)}$:
\begin{equation}
\mathbb{E}\!\left[f_i^{(1)}\mid d_i,y_i,S_i^{(0)}\right]
= G_i^{(0)}\,\mathbb{E}\!\left[f_i^{(0)}\mid y_i\right]
\end{equation}

The subgraph effect enters through the mixture
\begin{equation}
h^0_{i,\mathrm{eff}}=\beta_i^0 h^0_{i,S}+(1-\beta_i^0)h^0_{i,R}
\end{equation}
where $\beta_i^0$ is the aggregation share from the inductive subgraph $S$ (and $R=V\setminus S$).
On heterophilous graphs, $h^0_{i,S}$ and $h^0_{i,R}$ are often small; if the inductive subgraph is \emph{dominant}
(large $\beta_i^0$) and \emph{label-mismatched} (small $h^0_{i,S}$), then $h^0_{i,\mathrm{eff}}$ decreases and so does $G_i^{(0)}$,
weakening $\mathbb{E}[f_i^{(1)}\mid\cdot]$.
Thm~\ref{thm:theorem32prime} extends this to depth by giving a recursion with a layer-wise gain $G_i^{(\ell)}$:
\begin{equation}
\mathbb{E}\!\left[f_i^{(\ell+1)}\mid d_i,y_i,S_i^{(\ell)}\right]
=
G_i^{(\ell)}\,\xi_i^\ell\,\mathbb{E}\!\left[f_i^{(0)}\mid y_i\right]
\end{equation}

In heterophilous settings, $\bar r_i^\ell$ is typically non-positive. If the inductive subgraph remains influential across layers
(large $\rho_i^\ell$) and label-mismatched, then $\hat h^\ell_{i,\mathrm{eff}}$ becomes smaller, pushing
$\big((1+\rho_i^\ell)\hat h^\ell_{i,\mathrm{eff}}-\rho_i^\ell\big)$ toward a small or negative value and reducing $G_i^{(\ell)}$.
When $G_i^{(\ell)}<1$ over multiple layers, the expected class-related component decays multiplicatively with depth, leading to
systematic performance drops, especially for nodes whose computation graphs repeatedly overlap the inductive subgraph.

\section{Causal Analysis }
\label{sec:causal view}

\begin{figure}[ht]
    \centering\includegraphics[width=0.45\textwidth]{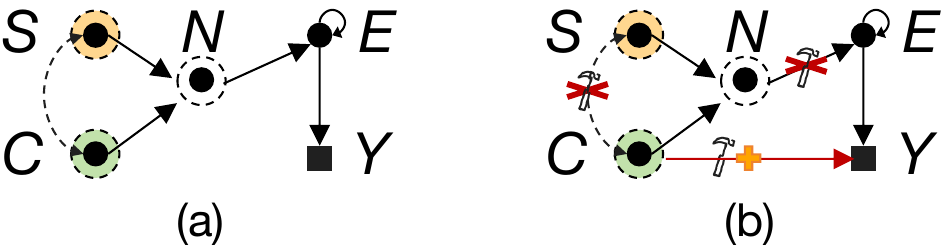}
    \caption{(a) Causal graph of existing GNN training. (b) Our proposed debiased causal graph.
    }
    \label{fig:scm}
\end{figure}

Causal learning has proven effective in uncovering model learning trajectories and exposing sources of model misrepresentation~\cite{pearl2009causality,wang2022causal,li2023causal, yu2023deconfounded}.
It has been widely applied to inspect model behaviors in various tasks, such as explanation~\cite{wang2024reinforced,yu2023counterfactual}, fairness~\cite{wang2024counterfactual,wang2024constrained} and bias mitigation~\cite{yu2025causal,huang2024counterfactual,wang2022off}.
Therefore, we adopt a causal perspective to analyze and correct biased learning behaviors caused by inductive subgraphs in the heterophilic setting. 
We first construct a Structure Causal Model (SCM)~\cite{pearl2000models} in Figure~\ref{fig:scm} (a) based on the following observations:
1) $\boldsymbol{S \rightarrow N \leftarrow C}$: 
There exist a shortcut subgraph $S$ and a causal subgraph $C$ within the neighborhood $N$, where $S$ captures easy-to-learn but spurious patterns, and $C$ encodes the true structural factors driving correct predictions.
2) $\boldsymbol{S \leftrightarrow C}$ : In heterophilic graphs, nodes from shortcut $S$ and causal $C$ subgraphs are often entangled, whereas in homophilic graphs, the homophily assumption enables separation.
For example, in strongly homophilic graphs such as Citeseer, well-defined clusters (e.g., math vs. science) enable clear separation. In contrast, heterophilic graphs exhibit more random structures~\cite{luan2024heterophilic} and lack such clusters.
3) $\boldsymbol{N \rightarrow E\circlearrowright \rightarrow Y}$: Neighbor messages from $N$ are aggregated into the ego node $E$, which predicts the label $Y$. $E$ is updated iteratively across layers, forming a self-loop. We now recall the $d$-connection theory:

\begin{definition}[$d$-connection~\cite{pearl2009causality}]
\label{def:d-connection}
$d$-connection indicates that information can flow between two variables, meaning they are conditionally dependent given a set of observed variables. 
Formally, in a causal graph $G$, $X$ and $Y$ are $d$-connected given a conditioning set $Z$ if there exists at least one path $p$ between $X$ and $Y$ such that $p$ is active (i.e., not blocked) with respect to $Z$, i.e., $X \not\!\perp\!\!\!\perp Y \mid Z$. 
\end{definition}

According to $d$-connection,
we could find two unblocked paths that would make the shortcut subgraph $S$ and the label $Y$ $d$-connected:
\begin{enumerate}
   \item  \textbf{Confounding Path} $C \leftarrow S \rightarrow N \rightarrow E \rightarrow Y$: $S$ acts as the central node between $C$ and $N$, and blocks the path $C \rightarrow \cdots Y$; thus $S$ is identified as a \textit{confounder}~\cite{pearl1998there} that influences both the causal variable $C$ and the outcome $Y$.
    The presence of the confounder $S$ introduces \textit{confounding bias}.
    Consequently, the shortcut $S$ and the prediction $Y$ are spuriously correlated, making GNNs easily learn the trivial patterns within $S$. 
   \item  \textbf{Spillover Path} $S \leftrightarrow C $: 
    The bidirectional relationship between shortcut neighbors $S$ and causal neighbors $C$ is defined as ``network interference''~\cite{yuan2021causal}. This interference creates a \textit{spillover effect}, where a node's outcome is influenced by the treatment assigned to its neighbors. This violates the Stable Unit Treatment Value Assumption (SUTVA), a common assumption in many GNN methods that states each node's outcome is independent of the treatment of others~\cite{chou2024state, sui2024invariant,ma2021causal}. 
\end{enumerate}

\noindent \textbf{Our Solution.}
Guided by this causal analysis, the core idea of our methodology is to intercept the confounding and spillover paths in order to make prediction $Y$ uncorrelated with shortcut $S$. As shown in Figure~\ref{fig:scm} (b): 
1) We establish the causal relationship between $C$ and $Y$, that is, $C \textcolor{red}{\rightarrow} Y$. In addition, we intercept $N \rightarrow E$ to block the confounding path, i.e., $C \leftarrow S \rightarrow N \Ccancel[red]\rightarrow E$. 
To do so, we should fully disentangle the shortcut subgraph $S$ and the causal subgraph $C$ from neighbors $N$, ensuring that predictions are based solely on the causal subgraph $C$.
2) We block the spillover path, i.e., $S \Ccancel[red]\leftrightarrow C$. This can be achieved by making $S$ and $C$ independent, i.e., $S \independent C$.

\section{The Proposed Causal Disentangled GNN}
\label{sec:methodology}
This section presents our proposed \textit{Causal Disentangled GNN} (CD-GNN), which mitigates confounding and spillover effects in heterophilic graph learning guided by our debiased causal graph.

\subsection{Causal Disentanglement for Heterophily}
We propose a novel causal disentanglement strategy, which blocks the \textbf{confounding path} by explicitly separating causal subgraph $C$ and shortcut subgraph $S$ from the neighborhood $N$ in heterophilic graphs.
The core idea is to employ generative probabilistic models to learn edge masks that assign edges to either $C$ or $S$.
These masks are optimized via backpropagation using two complementary loss functions: one encouraging the capture of shortcut information and the other enforcing the learning of causal information.

\subsubsection{Causal and Shortcut Masks}

Given a mini-batch of neighbor subgraphs $\mathcal{G}=\{N^1,\cdots, N^n\}$, in which a neighbor set $N^i \in \mathcal{G}$ is formulated as a subgraph $N^i=\{\mathbf{A}^i, \mathbf{X}^i\}$.
We initialize a structure mask $\mathbf{M}_a^i$ on the graph structure $\mathbf{A}^i$, and a feature mask $\mathbf{M}_x^i$ on the node feature $\mathbf{X}^i$.
Given a mask $\mathbf{M}$, its complementary mask is $\overline{\mathbf{M}}=\mathbb{I}-\mathbf{M}$, where $\mathbb{I}$ is the all-one matrix. 
Therefore, we can decompose the entire subgraph $N^i$ into two parts: $G_c=\left\{\mathbf{A}^i \odot \mathbf{M}_a^i, \mathbf{X}^i \odot \mathbf{M}_x^i\right\}$ and $G_s=\left\{\mathbf{A}^i \odot \overline{\mathbf{M}}_a^i, \mathbf{X}^i \odot \overline{\mathbf{M}}_x^i\right\}$, which correspond to the causal subgraph and the shortcut subgraph, respectively.
We then initialize causal and shortcut embedding of $G_c$ and $G_s$ by training a causal GNN ($\text{GNN}^L_c$) and a shortcut GNN ($\text{GNN}^L_s$), respectively:
\begin{equation}
\label{eq:embedding ini causal}
\mathbf{h}_{G_c}=f_{\text {readout }}\left(\text{GNN}^L_c\left(\mathbf{A}^i \odot \mathbf{M}_a^i, \mathbf{X}^i \odot \mathbf{M}_x^i\right)\right),
\end{equation}
\begin{equation}
\label{eq:embedding ini shortcut}
\mathbf{h}_{G_s}=f_{\text {readout}}\left(\text{GNN}^L_s\left(\mathbf{A}^i \odot \overline{\mathbf{M}}_a^i, \mathbf{X}^i \odot \overline{\mathbf{M}}_x^i\right)\right).
\end{equation}
where $\text{GNN}_c^L$ denotes an $L$-layer graph convolution applied to the causal subgraph $G_c$, and $\text{GNN}_s^L$ is defined analogously for the shortcut subgraph $G_s$.
$f_{\text {readout}}$ is the readout function that collects all node representations to form a graph representation.

\subsubsection{Graph Disentanglement}
We propose a novel algorithm for disentangling causal and shortcut subgraphs, guided by two key principles. 
First, \textbf{Shortcut Amplification in GNNs}: Motivated by our finding that shortcuts are easier to learn, the shortcut GNN ($\mathrm{GNN}_s^L$) is trained with a shortcut-amplification loss ($\mathcal{L}_s$). This loss increases the impact of shortcut subgraphs on the classification loss, forcing the model to focus on them. 
Second, \textbf{Causal Learning in GNNs}: In parallel, the causal GNN ($\mathrm{GNN}_c^L$) is trained with a causal-aware loss ($\mathcal{L}_c$). This loss uses a relative difficulty score to specifically target the more complex causal subgraphs, which we observed are harder for the model to learn.

\noindent \textbf{Shortcut Amplification.} 
As evidenced in our empirical study, shortcut information is easier to learn than causal information; therefore, a shortcut subgraph would result in a larger gradient compared with the causal subgraph for the loss function when training $\text{GNN}^L_s$.
Inspired by~\cite{nam2020learning}, we use the generalized cross-entropy (GCE) loss to amplify the gradient of the shortcut subgraph:
\begin{equation}\label{eq:shortcut loss}
    \mathcal{L}_{s} = \frac{1-\Phi_s^y\left(\mathbf{h};\theta_s\right)^q}{q}
\end{equation}
where $\Phi_s$ is the softmax classifier for $\text{GNN}^L_s$ and $\Phi_s^y$ is the probability distribution of samples (i.e., nodes) that belong to target class $y$. $\mathbf{h}=\left[\mathbf{h}_{G_c};\mathbf{h}_{G_s}\right]$ is the entangled embeddings of the causal and shortcut subgraphs. 
$\theta_s$ is the parameters of $\text{GNN}^L_s$.
$q \in(0,1]$ is a hyperparameter that controls the degree of gradient amplification. 

\noindent \textbf{Causal Learning.}
While GNNs more easily capture shortcut features, causal features—those most relevant to predictions—are inherently more challenging to learn. We define the relative difficulty score $\mathcal{W}(\mathbf{h})$ for learning causal versus shortcut subgraphs as:
\begin{equation}
    \mathcal{W}(\mathbf{h})=\frac{\operatorname{CE}\left(\Phi_s, y\right)}{\operatorname{CE}\left(\Phi_s, y\right)+\operatorname{CE}\left(\Phi_c, y\right)}
    \label{eq:wh}
\end{equation}
where $\operatorname{CE}$ denotes the cross-entropy loss, and $\Phi_c$ and $\Phi_s$ are the softmax classifiers for the causal GNN ($\text{GNN}^L_c$) and shortcut GNN ($\text{GNN}^L_s$), respectively.
$\mathbf{h}=\left[\mathbf{h}_{G_c};\mathbf{h}_{G_s}\right]$ is the entangled causal and shortcut embeddings.
We use $\mathcal{W}(\mathbf{h})$ to up-weight the loss of causal subgraphs, such that the model is incentivized to focus more on this critical information.
Thus, the objective function for training $\text{GNN}^L_c$ to capture causal subgraphs is defined as:
\begin{equation}\label{eq:causal loss}
    \mathcal{L}_c=\mathcal{W}(\mathbf{h}) \operatorname{CE}\left(\Phi_c^y\left(\mathbf{h};\theta_c\right)\right)
\end{equation}
where $\theta_c$ represents the parameters of $\text{GNN}^L_c$, $\Phi_c$ is the softmax classifier for $\text{GNN}^L_c$, and $\Phi_c^y$ denotes the probability distribution over class $y$.

\begin{lemma}
\label{lem:principles}
Consider the shortcut branch $\mathrm{GNN}_s$ trained with the generalized cross-entropy (GCE) shortcut loss $\mathcal{L}_s$ in Eq.~\eqref{eq:shortcut loss}
and the causal branch $\mathrm{GNN}_c$ trained with the reweighted causal loss $\mathcal{L}_c$ in Eq.~\eqref{eq:causal loss}.
Then:
\begin{enumerate}[leftmargin=*]
\item \textbf{Shortcut-gradient amplification.}
The gradient of $\mathcal{L}_s$ is a confidence-reweighted CE gradient:
\begin{equation}
\frac{\partial \mathcal{L}_s}{\partial \theta_s}
=
\big(\Phi_s^{y}\big)^{q}
\cdot
\frac{\partial\, \mathrm{CE}(\Phi_s,y)}{\partial \theta_s}
\label{eq:gce_reweight_optB}
\end{equation}
where $\Phi_s^{y}$ is the predicted probability of the target class and $q\in(0,1]$ controls the reweighting strength.
Thus samples with larger shortcut confidence receive larger gradients, so GNN$_s$ learns shortcut patterns in $G_s$ faster.

\item \textbf{Causal-priority reweighting.}
The weight $\mathcal{W}(h)$ in Eq.~\eqref{eq:causal loss} increases the loss contribution of samples that are relatively harder for the causal branch.
When shortcut patterns are easier to fit than causal patterns, this reweighting encourages GNN$_c$ to focus on causal information
instead of relying on shortcut cues.
\end{enumerate}
\end{lemma}

\subsection{Spillover Effect Controlled Optimization}\label{sec:spillover}

To block the \textbf{spillover path} $S \leftrightarrow C$, we enforce independence between causal embeddings $\mathbf{h}_{G_c}$ and shortcut embeddings $\mathbf{h}_{G_s}$.
Inspired by~\cite{fan2022debiasing}, a simple yet effective approach is to construct counterfactual shortcut embeddings, which by design lack any factual correlations with causal embeddings.~\footnote{Causal embeddings are not perturbed, as they encode the causal structure.}
In particular, we construct counterfactual shortcut embeddings $\hat{\mathbf{h}}_{G_s}$ by randomly permuting $\mathbf{h}_{G_s}$ within each mini-batch, yielding counterfactual embeddings $\mathbf{h}_{ct}=[\mathbf{h}_{G_c};\hat{\mathbf{h}}_{G_s}]$,  where $\hat{\mathbf{h}}_{G_s}$ is naturally decorrelated from $\mathbf{h}_{G_c}$.
The labels are permuted accordingly to obtain counterfactual labels $y^{cf}$.
We then define the counterfactual loss $\mathcal{L}_{cf}$ for training the causal and shortcut GNNs:
\begin{equation}
\label{eq:loss swap}
   \mathcal{L}_{cf} =  \frac{1-\Phi_s^{y^{cf}}\left(\mathbf{h}_{ct};\theta_s\right)^q}{q} + \mathcal{W}(\mathbf{h}) \operatorname{CE}\left(\Phi_c^y\left(\mathbf{h}_{ct};\theta_c\right)\right)
\end{equation}

Till now, $\mathcal{L}_{cf}$ reduces global correlations between causal and shortcut embeddings.
However, since causal and shortcut embeddings comprise multiple node-level embeddings, individual node embeddings may still remain entangled.
To enforce node-level disentanglement, we introduce Hilbert–Schmidt Independence Criterion (HSIC)~\cite{gretton2005measuring} as a regularizer. HSIC is a principled statistical test of independence, which we apply at the output layer to enforce independence between causal and shortcut node embeddings.

Formally, let $h_{G_c}$, $h_{G_s}$ be a node embedding within causal embeddings $\mathbf{h}_{G_c}$ and shortcut embeddings $\mathbf{h}_{G_s}$, respectively.
We define $\phi\left(h_{G_c}\right)$ and $\psi\left(h_{G_s}\right)$ as mapping functions that map $h_{G_c}$ and $h_{G_s}$ to kernel space $\mathcal{F}$ and $\Omega$, respectively. 
$\mathcal{F}$ and $\Omega$ are Reproducing Kernel Hilbert Space (RKHS) on the original space of causal embeddings $\mathbf{h}_{G_c}$ and shortcut embeddings $\mathbf{h}_{G_s}$, respectively. 
Then, the HSIC between causal embedding $h_{G_c}$ and shortcut embedding $h_{G_s}$ is defined as the squared Hilbert–Schmidt norm of their cross-covariance operator:
\begin{equation}
\begin{aligned}
   &\operatorname{HSIC}(\mathbb{P}_{h_{G_c}, h_{G_s}})
    := \\
& \left\| \mathbb{E}_{h_{G_c}, h_{G_s}} 
\big[ (\phi(h_{G_c}) - \mu_{h_{G_c}}) 
\otimes (\psi(h_{G_s}) - \mu_{h_{G_s}}) \big] \right\|_{HS}^2 
\end{aligned}
\end{equation}
where $\otimes$ denotes the tensor product, $\phi(\cdot)$ and $\psi(\cdot)$ map embeddings to the RKHS $\mathcal{F}$ and $\Omega$, respectively, 
$\mu_{h_{G_c}}$ and $\mu_{h_{G_s}}$ are their mean embeddings. $\mathbb{P}_{h_{G_c}, h_{G_s}}$ is the probability distribution for sampling pairs of $h_{G_c}$, $h_{G_s}$, and $\|X\|_{HS}^2 = \sum_{i,j} x_{i,j}^2$.

Accordingly, node embedding independence loss between causal embeddings $\mathbf{h}_{G_c}$ and shortcut embeddings $\mathbf{h}_{G_s}$ is calculated by,
\begin{equation}
\label{eq:HSIC}
    \mathcal{L}_{HSIC}=\sum_{h_{G_c} \in \mathbf{h}_{G_c}} \sum_{h_{G_s} \in \mathbf{h}_{G_s}} \operatorname{HSIC}\left(h_{G_c}, h_{G_s},\mathcal{F}, \Omega\right)
\end{equation}

\noindent \textbf{Model Optimization.}
Combined with the graph disentanglement process, the learning objective of our model is formulated as, 
\begin{equation}
\mathcal{L}:=\underbrace{\mathcal{L}_{s}}_{\text{shortcut}}+\underbrace{\mathcal{L}_{c}}_{\text{causal}}+\underbrace{\lambda_1\mathcal{L}_{cf}+\lambda_2\mathcal{L}_{HSIC}}_{\text{spillover control}}
\label{eq:fullloss}
\end{equation}
where $\mathcal{L}_{s}$ as in Eq.~\eqref{eq:shortcut loss} and $\mathcal{L}_{c}$ as in Eq.~\eqref{eq:causal loss} are the primary loss terms for the disentanglement of shortcut and causal subgraphs, respectively. 
$\mathcal{L}_{cf}$ in Eq.~\eqref{eq:loss swap} and $\mathcal{L}_{HSIC}$ in Eq.~\eqref{eq:HSIC} acts as regularizers and their strength should be adjusted according to different learning scenarios, we thus incorporate the hyperparameters $\lambda_1$ and $\lambda_2$ for them for adaptive learning purpose.

\subsection{Theoretical Guarantees for Mitigating Shortcut Bias}
\label{sec:cdgnn_improves_c4c5}

Our goal is to show that the proposed disentanglement method (Eqs.~\eqref{eq:embedding ini causal}-~\eqref{eq:HSIC}) mitigates the degradation described in
Thm~\ref{thm:theorem31prime}-\ref{thm:theorem32prime}. The key idea is that, after training, the \emph{causal branch} relies much less on the shortcut (inductive subgraph)
messages. This reduces the shortcut dominance ratios in the causal branch, which increases the effective homophily
$h_{i,\mathrm{eff}}$ (one layer) and $\hat h^{\ell}_{i,\mathrm{eff}}$ (deep layers). Since larger effective homophily
directly increases the layer-wise gains $G_i^{(0)}$ and $G_i^{(\ell)}$ in Thm~\ref{thm:theorem31prime}-\ref{thm:theorem32prime}, so the class-related signal decays more slowly
with depth, leading to improved node classification performance on heterophilous graphs.

\begin{assumption}
\label{ass:effective_disentangle}
After optimizing the full objective in Eq.~\eqref{eq:fullloss}, the causal representation is weakly dependent on the shortcut representation
and the causal prediction is insensitive to shortcut perturbations. Specifically, there exists $\varepsilon>0$ such that:
\begin{enumerate}[leftmargin=*]
\item Independence: $\mathrm{HSIC}(h_{G_c},h_{G_s}) \le \varepsilon$ (equivalently, $\mathcal{L}_{HSIC}$ is small).
\item Counterfactual invariance: for random permutations $\pi(\cdot)$ of shortcut embeddings used in Eq.~\eqref{eq:loss swap},
\[
\big|\Phi_c^{y}(h_{G_c},h_{G_s})-\Phi_c^{y}(h_{G_c},\pi(h_{G_s}))\big|\le \varepsilon 
\]
\item Small shortcut dominance in the causal branch: for each layer/hop $\ell$ and affected node $i$,
\[
\beta_{i,S}^{(\ell)} \le \varepsilon_\ell
\quad\text{and}\quad
\rho_{i,S}^{\ell}\le \varepsilon_\ell, \quad \varepsilon_\ell \text{  is small}
\]
\end{enumerate}

Moreover, the shortcut subgraph $G_s$ captures the inductive/recurring motif responsible for the degradation in
Thm~\ref{thm:theorem31prime} and~\ref{thm:theorem32prime}, and for affected nodes it is label-inconsistent:
\[
h_{i,S}^{0}\le h_{i,R}^{0}-\Delta
\quad\text{for some }\Delta>0
\]
\end{assumption}

\begin{theorem}[Disentanglement increases the Thms~\ref{thm:theorem31prime}-~\ref{thm:theorem32prime} gain factors]
\label{thm:improve_gains}
Consider affected nodes on a heterophilous graph where Thms~\ref{thm:theorem31prime} and ~\ref{thm:theorem32prime} predict degradation due to a dominant inductive subgraph $S$.
Under Assumptions~\ref{ass:effective_disentangle}, the causal branch achieves a larger effective homophily:
\begin{equation}
\begin{split}
    h^0_{i,\mathrm{eff}}(\text{causal})&=\beta_{i,S}^{(0)} h^0_{i,S} + (1-\beta_{i,S}^{(0)}) h^0_{i,R}\\
&\;\ge\;
h^0_{i,\mathrm{eff}}(\text{baseline}) + \Delta\cdot(\beta_i^0-\beta_{i,S}^{(0)}) - O(\varepsilon)
\end{split}
\label{eq:heff_improve}
\end{equation}
where $\beta_i^0$ is the baseline dominance ratio (without disentanglement) and $\beta_{i,S}^{(0)}\le \varepsilon_0$.
As a result, the one-layer gain in Thm~\ref{thm:theorem31prime} improves:
\begin{equation}
G_i^{(0)}(\text{causal}) \;\ge\; G_i^{(0)}(\text{baseline}) + c_0\big(\beta_i^0-\varepsilon_0\big)\Delta - O(\varepsilon)
\label{eq:G0_improve}
\end{equation}
for some constant $c_0>0$ determined by $W_i^0,\rho,$ and $d_i$.

Moreover, at depth $\ell$ the same argument yields a larger deep-layer effective homophily $\hat h^\ell_{i,\mathrm{eff}}$ and
a smaller shortcut dominance ratio $\rho_{i,S}^{\ell}\le \varepsilon_\ell$, which increases the deep-layer gains:
\begin{equation}
G_i^{(\ell)}(\text{causal}) \;\ge\; G_i^{(\ell)}(\text{baseline}) + c_\ell\big(\rho_i^\ell-\varepsilon_\ell\big)\Delta_\ell - O(\varepsilon)
\label{eq:Gell_improve}
\end{equation}
with constants $c_\ell>0$ and layer-dependent gaps $\Delta_\ell>0$.
Therefore, the expected class-related component is preserved better across layers:
\begin{equation}
\begin{split}
    &\Big\|\mathbb{E}[f_i^{(\ell)}\mid \cdot]\Big\|_{\text{causal}}
\;\ge\;
\Big(\prod_{t=0}^{\ell-1}\frac{G_i^{(t)}(\text{causal})}{G_i^{(t)}(\text{baseline})}\Big)
\,
\Big\|\mathbb{E}[f_i^{(\ell)}\mid \cdot]\Big\|_{\text{baseline}} \\
    &\text{where the ratio is } > 1 \text{ when gains improve.}
\end{split}
\label{eq:signal_preserve}
\end{equation}
\end{theorem}

\begin{proof}[Proof]
In Thm~\ref{thm:theorem31prime}, the harmful effect of the inductive subgraph enters only through $h^0_{i,\mathrm{eff}}$ in Eq.~\eqref{eq:C4_heff_simple}, where $h^0_{i,S}$ is small on heterophilous graphs.
If the method reduces shortcut dominance in the causal branch from $\beta_i^0$ to $\beta_{i,S}^{(0)}\le\varepsilon_0$,
then the mixture weight on the small term $h^0_{i,S}$ decreases, so $h^0_{i,\mathrm{eff}}$ increases by at least
$\Delta(\beta_i^0-\beta_{i,S}^{(0)})$ up to $O(\varepsilon)$ terms, giving Eq.~\eqref{eq:heff_improve}.
Since $G_i^{(0)}$ is monotone in $h^0_{i,\mathrm{eff}}$ in Eq.~\eqref{eq:C4_gain_simple} and ~\eqref{eq:G0_improve} follows. In Thm~\ref{thm:theorem32prime}, the harmful effect enters through the layer-wise gain
$G_i^{(\ell)}$ in Eq.~\eqref{eq:C5_gain_simple}, which depends on $\rho_i^\ell$ and $\hat h^\ell_{i,\mathrm{eff}}$.
The independence and counterfactual objectives (Eqs.~\eqref{eq:loss swap}, ~\eqref{eq:HSIC}) reduce shortcut dependence of the causal predictor
(Lemmas~\ref{lem:principles}), implying smaller shortcut dominance
$\rho_{i,S}^{\ell}\le\varepsilon_\ell$ and thus a larger effective homophily term, which increases $G_i^{(\ell)}$ and yields Eq.~\eqref{eq:Gell_improve}. The multiplicative preservation Eq.~\eqref{eq:signal_preserve} follows by iterating Eq.~\eqref{eq:C5_gain_simple}.
\end{proof}

Under the conditions of Thm.~\ref{thm:improve_gains}, the causal branch produces stronger expected class-related signals across layers. By standard margin-based generalization arguments, where a larger correct-class signal corresponds to a larger classification margin, this results in reduced misclassification on the affected nodes. Consequently, the proposed method alleviates the performance degradation predicted by Thm.~\ref{thm:theorem31prime} and Thm.~\ref{thm:theorem32prime}.

\subsection{Time Complexity and Memory Complexity}
Consider a graph with $N=|\mathcal{V}|$ nodes and $E=|\mathcal{E}|$ edges. 
Our CD-GNN framework primarily relies on two GNN classifiers: the causal GNN ($\text{GNN}^L_c$) and the shortcut GNN ($\text{GNN}^L_s$) to capture causal and shortcut information, respectively.
Let $L$ be the number of GNN layers, $d$ the graph embedding dimension. 
Training the two GNN classifiers dominates the forward/backward cost. Under standard per-layer message passing assumptions, the per-epoch time complexity for one GNN is $\mathcal{O}(L(E\cdot d + N\cdot d^2))$. Training both classifiers (two GNNs) doubles the cost, so the training cost per epoch is $ \mathcal{O}(2L(E\cdot d + N\cdot d^2))$.
The memory complexity for one GNN mainly arises from storing node embeddings, intermediate activations, and adjacency structures. It scales as $\mathcal{O}(N d + E + L N d)$, where $Nd$ accounts for the embedding matrix, $E$ for the sparse adjacency representation, and $LNd$ for cached activations during backpropagation. 
For two GNN classifiers, this becomes
$\mathcal{O}(2(N d + E + L N d))$.
After obtaining the node embeddings from both GNNs, we compute the Hilbert–Schmidt Independence Criterion (HSIC) to measure the statistical independence.
The naive (full Gram-matrix) HSIC computation requires evaluating all pairwise kernel similarities, leading to time complexity of $\mathcal{O}(N^2 d)$ and memory complexity of $\mathcal{O}(N^2)$.
In total, our CD-GNN has a computation complexity of $\mathcal{O}_{all} = \mathcal{O}(2L(E\cdot d + N\cdot d^2)) + \mathcal{O}(N^2 d)$, and a memory complexity of $\mathcal{M}_{all} = \mathcal{O}(2(N d + E + L N d)) + \mathcal{O}(N^2)$. 
Thus, the overall time and memory complexity of CD-GNN does not exceed $\mathcal{O}(N^2)$, which is acceptable for the typical graph scales used in practice. 

\section{Experiments}
We adopt seven benchmark graph datasets for evaluation, including Chameleon, Squirrel, Cornell, Roman-empire, Amazon-ratings, Computers, and Questions. The first five datasets are highly heterophilious, whereas the last two datasets tend to be homophilious.
As reported in~\cite{luan2024heterophilic}, these five datasets are among the most challenging heterophilic benchmarks, exhibiting high degrees of heterophily.
We compare CD-GNN with ten baselines across three categories: \textbf{Standard GNNs}: GCN~\cite{kipf2016semi}, GAT~\cite{velivckovic2017graph} and GraphSAGE~\cite{hamilton2017inductive}, recent \textbf{Causality‐inspired models}: CIE~\cite{chen2023causality} (Plug-in GAT and GCN yielding variants: $\text{CIE}_{\text{GAT}}$ and $\text{CIE}_{\text{GCN}}$ ) and CAT~\cite{he2024cat}, and 
\textbf{SOTA heterophilic graph learning models}: FAGCN~\cite{bo2021beyond}, CAGNN~\cite{chen2023exploiting}, GGCN~\cite{yan2022two} and LatGRL~\cite{shen2025heterophily}.
To ensure fair comparisons, all baselines follow the same data splitting strategy: 60\% of the data is used for training, 20\% for validation, and 20\% for testing. Each method is run three times, and the average results and standard deviations are reported in Table~\ref{tab:overall}. For our method, we used the Adam optimizer for optimization. We set the learning rate, weight decay, the number of neurons in the hidden layer, and the dropout rate to 0.0001, 0.0005, 150, and 0.1, respectively. We select the best configurations of hyperparameters of all baseline methods and our proposed model through grid search.
Code of our model is available at:\href{https://github.com/Chrystalii/CD-GNN}{
https://github.com/Chrystalii/CD-GNN}

\begin{table*}
\caption{Comparison of classification accuracy ($\pm$ means standard deviation). In general, CD-GNN outperforms the others on 6 out of 7 datasets. The best results per benchmark are highlighted in bold and runner-ups are highlighted in underline. OOM means Out of GPU Memory.
Most improvements against the second-best results are significant at $p < 0.01$.
}\label{tab:overall}
\centering
\setlength{\tabcolsep}{5pt}
\renewcommand{\arraystretch}{1.2}
\resizebox{\textwidth}{!}{
\begin{tabular}
    {c ccc cccc } \toprule
Dataset & \textbf{Chameleon} & \textbf{Squirrel}& \textbf{Cornell}  & \textbf{Roman-empire} & \textbf{Amazon-ratings} & \textbf{Computers} & \textbf{Questions} \\
$h_L$  & 0.7639 & 0.7928 & 0.8773 & 0.9431& 0.6196 & 0.2289  & 0.1384\\
$h_F$ & 0.7559 & 0.8095 & 0.8890 & 0.9540 & 0.6243 & 0.3742  & 0.1020 \\
\#Nodes & 2,277 & 5,201 & 183 & 22,662 & 24,492 & 13,752 & 48,921  \\
\#Edges & 36,101 & 198,493  & 298 & 32,927 & 93,050 & 491,722 & 153,540\\
\#Node labels & 5 & 5 & 5 & 18 & 5 & 10 & 2 \\
\midrule 

\midrule
GCN &
0.6031 ($\pm$0.0143) &\underline{0.5024 ($\pm$0.0088)} &0.4865 ($\pm$0.0584) &0.3445 ($\pm$0.0069) &0.4195 ($\pm$0.0009)&  0.6480 ($\pm$0.1935) &0.9718 ($\pm$0.0000) \\
GAT &0.5504 ($\pm$0.0109) &0.4236 ($\pm$0.0130) &0.4955 ($\pm$0.0127) &0.3936 ($\pm$0.0033) &0.3893 ($\pm$0.0156) &\underline{0.8447 ($\pm$0.0430)} &0.9718 ($\pm$0.0000) \\
GraphSAGE & \underline{0.6301 ($\pm$ 0.0109)} & 0.4941($\pm$0.0068) & 0.6667 ($\pm$0.0674) & 0.6148 ($\pm$0.0062) & 0.4151 ($\pm$0.0028) & 0.5433 ($\pm$0.0392) & 0.9718 ($\pm$0.0000) \\
\midrule
\midrule
$\text{CIE}_{\text{GAT}}$ &0.5877 ($\pm$0.0090) &0.4348 ($\pm$0.0050) &0.4775 ($\pm$0.0255) &OOM &OOM & 0.3796 ($\pm$0.0071) &OOM\\
$\text{CIE}_{\text{GCN}}$ &0.5804 ($\pm$0.0090) &0.4672 ($\pm$0.0076) &0.4414 ($\pm$0.0337) &OOM &OOM & 0.4271 ($\pm$0.0735) &OOM\\
CAT &0.5329 ($\pm$0.0220) &0.4172 ($\pm$0.0086) &0.7297 ($\pm$0.0221)&0.5237 ($\pm$0.0053) &0.3750 ($\pm$0.0014) &0.7111 ($\pm$0.0183) &\underline{0.9719 ($\pm$0.0000)}\\
\midrule
\midrule
FAGCN & 0.4751 ($\pm$0.0102) & 0.3871 ($\pm$0.0163) & 0.6757 ($\pm$0.0584) & 0.4734 ($\pm$0.0107) & 0.3930 ($\pm$0.0140) & 0.4957 ($\pm$0.1015) & 0.9718 ($\pm$0.0000) \\
CAGNN & 0.5329 ($\pm$0.0234) & 0.4076 ($\pm$0.0106) & \hl{\textbf{0.7928} ($\pm$0.0337)} & 0.5451 ($\pm$0.0139) & 0.3745 ($\pm$0.0011) & 0.5592 ($\pm$0.1669) &  0.9718 ($\pm$0.0000)\\
GGCN & 0.4912 ($\pm$0.0159) & 0.2763 ($\pm$0.0349) & 0.6937 ($\pm$0.0127) &\underline{0.6528 ($\pm$0.0055)} & \underline{0.4247 ($\pm$0.0049)} & 0.5604 ($\pm$0.1442) & 0.9712 ($\pm$0.0007)  \\
LatGRL & 0.5702 ($\pm$0.0093) & 0.3807 ($\pm$0.0053) & 0.6216 ($\pm$0.0221) & 0.3010 ($\pm$0.0705) & 0.3838 ($\pm$0.0091) & 0.6461 ($\pm$0.0300) & OOM \\
\midrule
\midrule
\textbf{CD-GNN} &\hl{\textbf{0.6762 ($\pm$0.0314)}}  &\hl{\textbf{0.5898 ($\pm$0.0230)}} &\underline{0.7297 ($\pm$0.0709)} &
\hl{\textbf{0.6585 ($\pm$ 0.0039)}} &\hl{\textbf{0.4451 ($\pm$0.0247)}} &\hl{\textbf{0.8848($\pm$0.0358)}} &\hl{\textbf{0.9720 ($\pm$0.0002)}}\\
\bottomrule
\end{tabular}
}
\end{table*}

\subsection{Node Classification Performance}
Table~\ref{tab:overall} shows the node classification accuracy of CD‐GNN and ten baselines on seven datasets. CD‐GNN ranks first on Chameleon, Squirrel, Roman-empire, Amazon-ratings, Computers and Questions, and runner-up on Cornell.
The improved performance across six datasets suggests the effectiveness of the proposed CD-GNN in both heterophilic and homophilic settings. The comparatively lower ranking on the Cornell dataset may be related to its small scale ($183$ nodes and $298$ edges), which offers limited information for the model to fully exploit causal–shortcut disentanglement mechanisms.
We analyze baselines to explain the improved performance of our CD-GNN.
Vanilla GCN and GraphSAGE aggregate neighbor information via uniform or mean pooling, and thus cannot adaptively identify the causal drivers of node labels.
In contrast, CD-GNN disentangles causal subgraphs and propagates causal signals, yielding substantially better performance.
Although GAT weights neighbor messages via attention, it generally ranks last among standard GNNs.
This suggests its attention scores become generally unreliable in heterophilous graphs.
In contrast, CD-GNN does not rely on attention and maintains robust performance.
$\text{CIE}_{\text{GAT}}$ and $\text{CIE}_{\text{GCN}}$ mitigate spurious correlations via backdoor adjustment, attributing performance degradation to unobserved confounders (external to the model).
However, they overlook that the dominant source of bias in heterophilic graphs arises from observed confounders within the graph structure itself.
In contrast, CD-GNN explicitly models inductive subgraphs as observed confounders, and effectively blocks both confounding and spillover paths during inference.
Besides, $\text{CIE}_{\text{GAT}}$ and $\text{CIE}_{\text{GCN}}$ incur substantial computational overhead and run out of memory when applied to three large-scale datasets.
By selecting similar neighbors using causal effects, CAT alleviates the distraction effect and performs the second-best on the least-heterophilous Questions dataset. However, as heterophily increases, CAT falls behind our CD‐GNN. This shows that while causal effects are beneficial, robust performance in highly heterophilous graphs requires explicit causal-shortcut disentanglement and causal signal propagation as in our CD-GNN.
FAGCN and LatGRL rely on separating low- and high-frequency graph signals, while CAGNN and GGCN adjust edge weights based on neighbor effects or node relative degree. However, both graph signals and node-level properties are intrinsic, correlation-driven indicators of heterophily.
They fail to address the root cause: the presence of inductive subgraphs that induce spurious correlations in heterophilic graphs. 
CD-GNN outperforms these baselines by explicitly mitigating confounding and spillover effects arising from shortcut inductive subgraphs, achieving improved performance.

\subsection{Ablation Study and Parameter Analysis }
We explore how key designs, including Shortcut Amplification ($\mathcal{L}_s$), Causal Learning ($\mathcal{L}_c$), Counterfactual Learning ($\mathcal{L}_{cf}$), and Hilbert-Schmidt Independence Criterion ($\mathcal{L}_{HSIC}$), impact the performance of our proposed CD-GNN model.
As shown in Figure~\ref{fig:ablation}, removing any one of these components leads to a noticeable drop in accuracy.
We analyze the effectiveness of each key component as below.
The Shortcut Amplification loss $\mathcal{L}_s$ explicitly steers CD‐GNN to amplify the easier‐to‐learn subgraphs (the shortcuts) by assigning larger gradients to high‐confidence, easy examples. This forces the shortcut GNN ($\text{GNN}^L_s$) to rapidly capture and isolate shortcut signals, preventing them from mixing with more causal subgraphs.
The Causal Learning loss $\mathcal{L}_c$ is key to CD-GNN's success. By using a difficulty weight $\mathcal{W}(\mathbf{h})$, which is larger for truly causal subgraphs, we explicitly up-weight the cross-entropy loss on complex yet informative causal patterns.
The Counterfactual loss $\mathcal{L}{cf}$ enforces sample-level decorrelation by intervening on shortcut embeddings and preventing their influence on causal predictors. The HSIC loss $\mathcal{L}{HSIC}$ enforces node-level independence across all embeddings. Ablation studies show that removing either loss reduces accuracy, highlighting the necessity of both for isolating true causal subgraphs and ensuring robust generalization.

\begin{figure}[!ht]
\centering
\begin{minipage}[t]{0.23\textwidth}
\centering
\includegraphics[width=\textwidth]{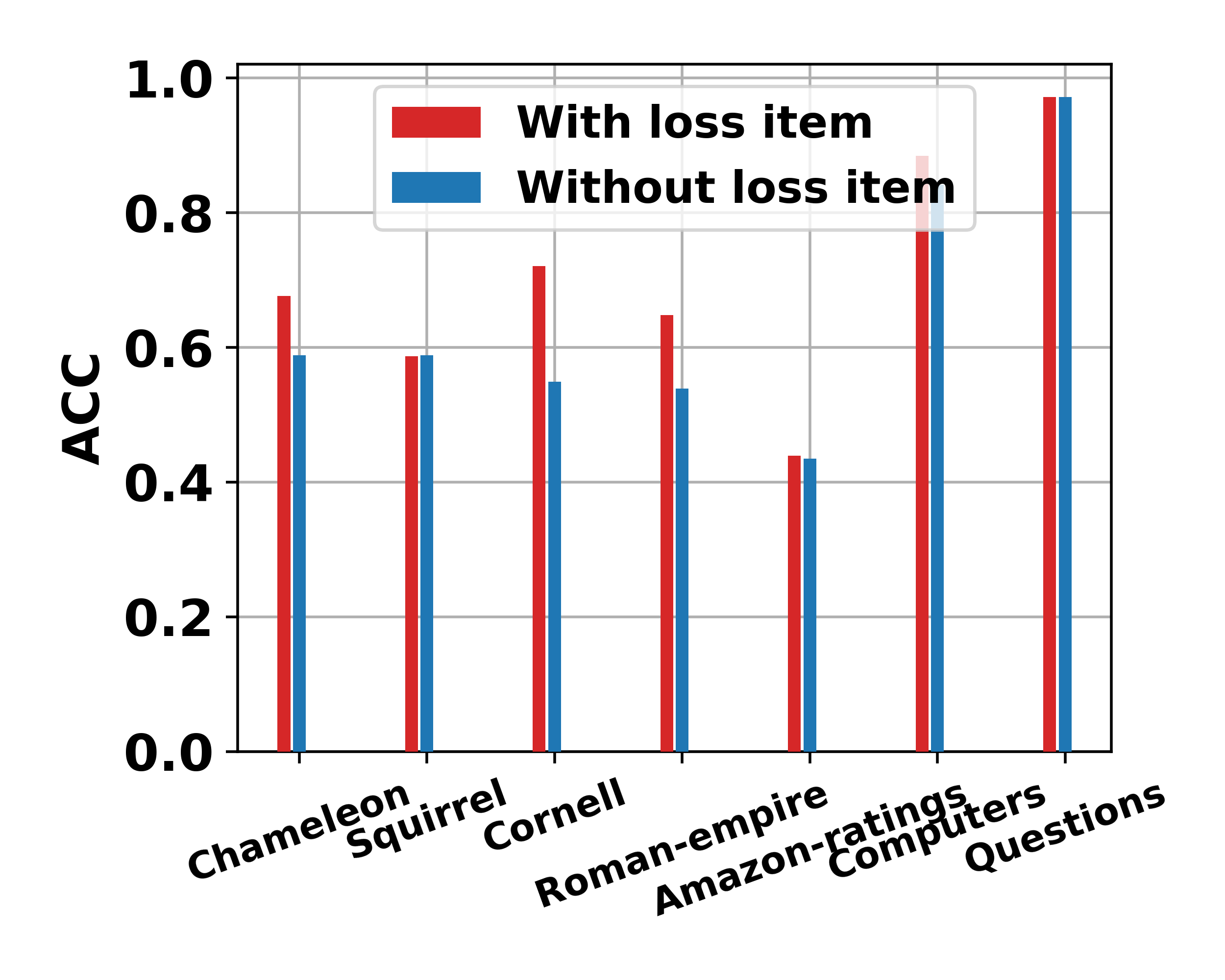}
\subcaption{Ablate $\mathcal{L}_{s}$}
\end{minipage}
\begin{minipage}[t]{0.23\textwidth}
\centering
\includegraphics[width=\textwidth]{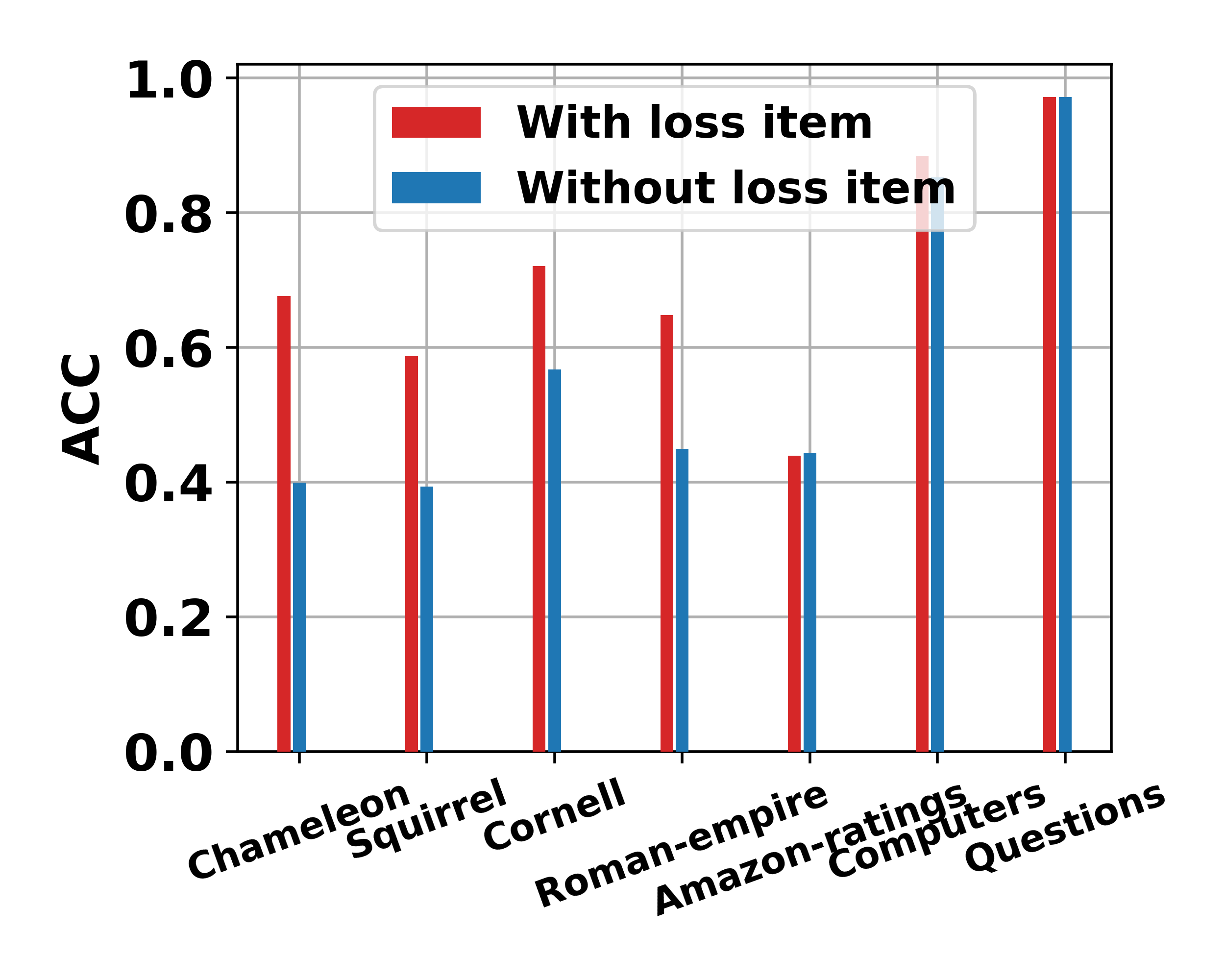}
\subcaption{Ablate $\mathcal{L}_{c}$}
\end{minipage}
\begin{minipage}[t]{0.23\textwidth}
\centering
\includegraphics[width=\textwidth]{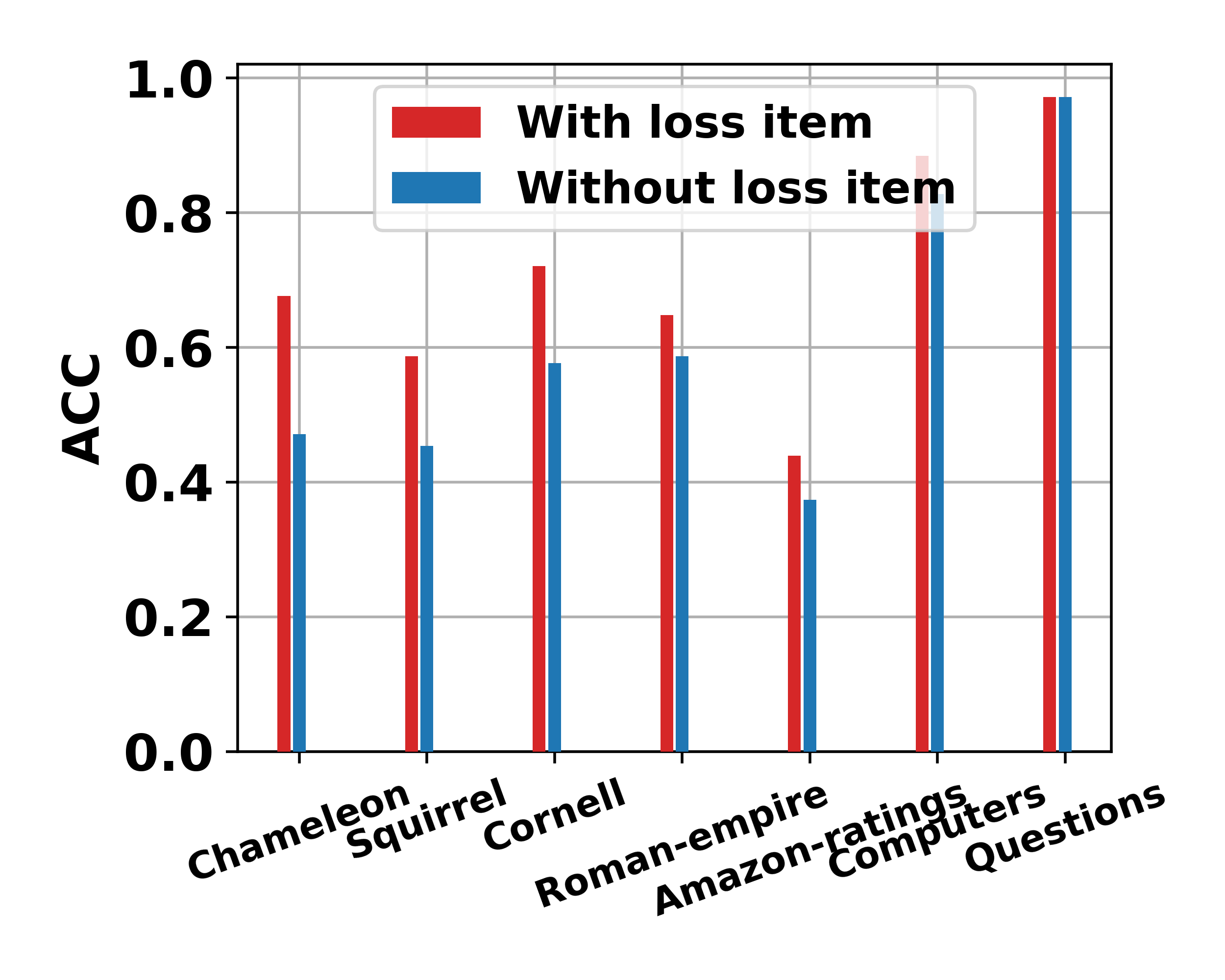}
\subcaption{Ablate $\mathcal{L}_{cf}$}
\end{minipage}
\begin{minipage}[t]{0.23\textwidth}
\centering
\includegraphics[width=\textwidth]{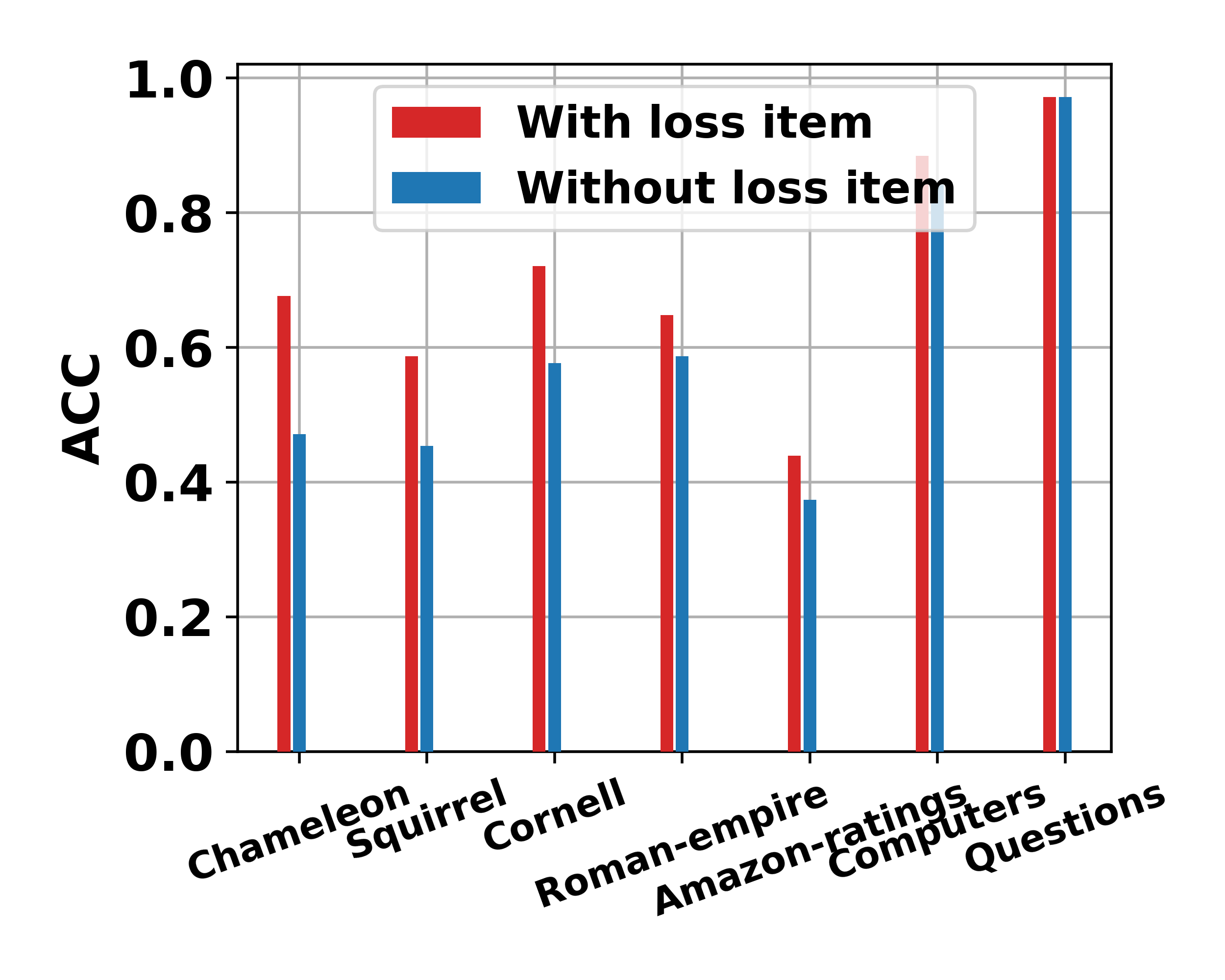}
\subcaption{Ablate $\mathcal{L}_{HSIC}$}
\end{minipage}
\caption{Ablation studies on seven datasets.
}
\label{fig:ablation}
\end{figure}

\begin{figure}[!ht]
\centering
\begin{minipage}[t]{0.4\textwidth}
\centering
\includegraphics[width=\textwidth]{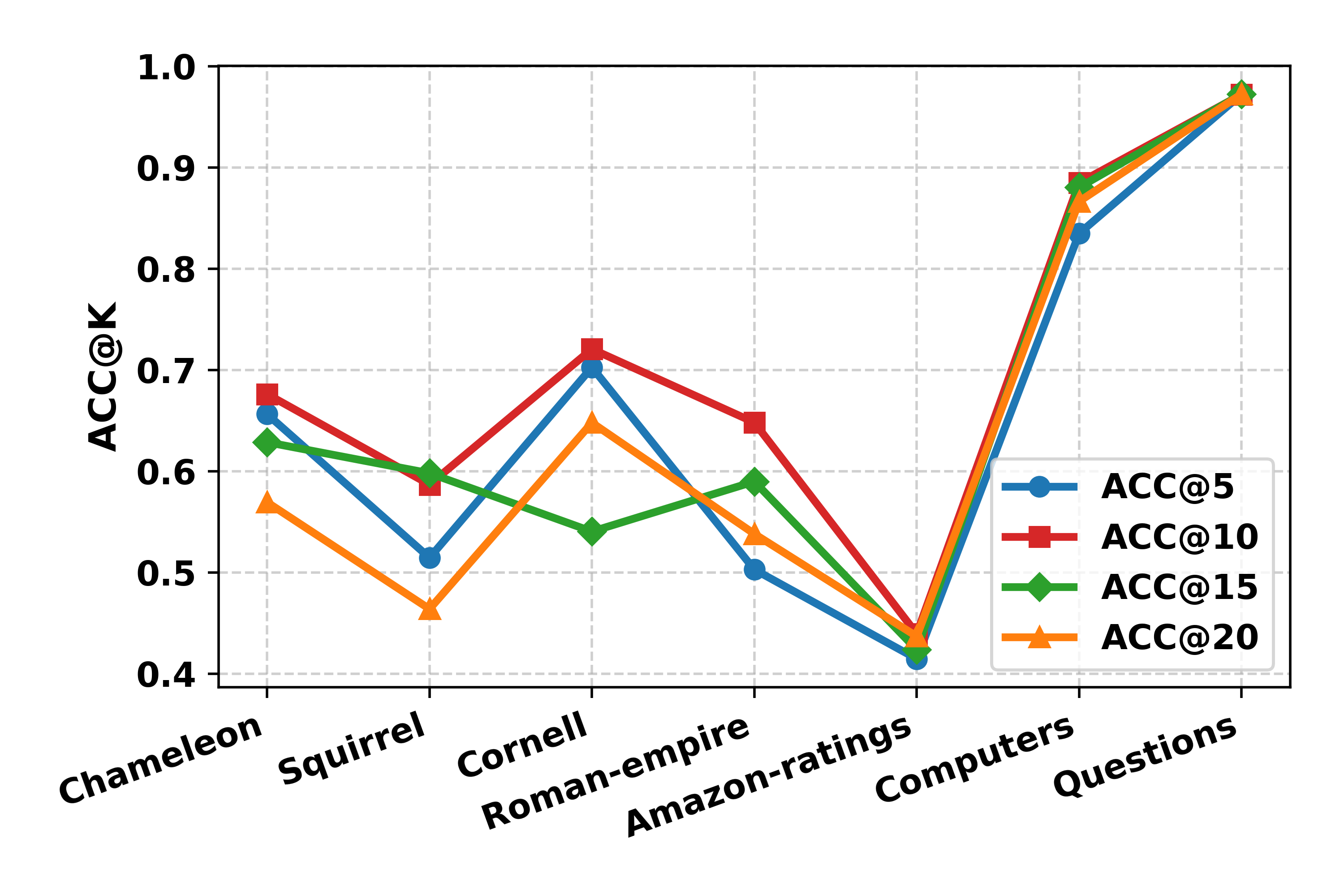}
\subcaption{Impact of $\lambda_1$ }
\end{minipage}
\begin{minipage}[t]{0.4\textwidth}
\centering
\includegraphics[width=\textwidth]{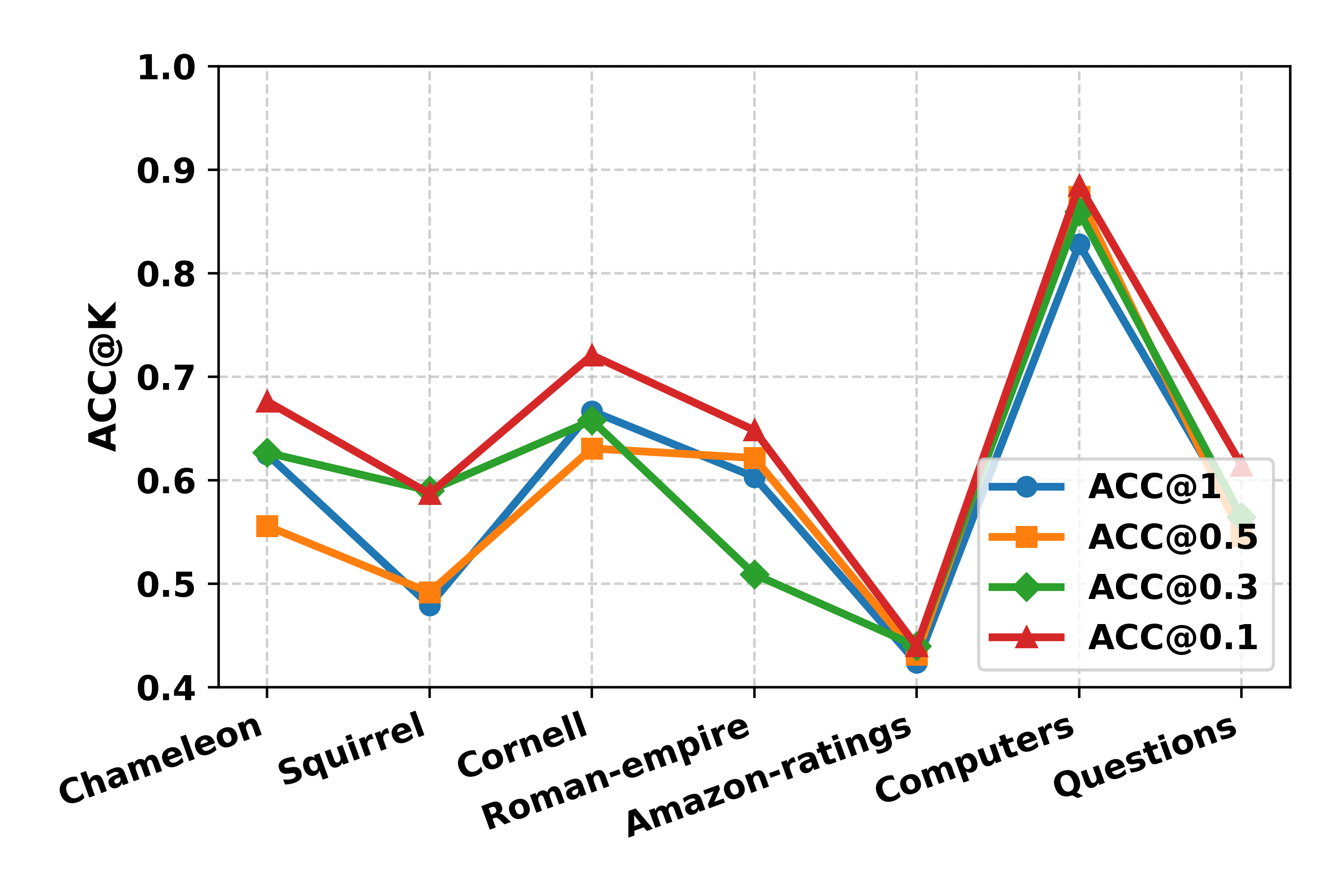}
\subcaption{Impact of $\lambda_2$}
\end{minipage}
\caption{Parameter analysis on seven datasets.
}
\label{fig:parameter}
\end{figure}

Figure~\ref{fig:parameter} (a) (b) report the parameter sensitivity of our CD-GNN w.r.t. $\lambda_1$ and $\lambda_2$, which control the impact degree of $\mathcal{L}_{cf}$ and $\mathcal{L}_{HSIC}$, respectively.
Based on existing studies~\cite{fan2022debiasing,gretton2005measuring}, $\lambda_1 $ is chosen from $\{5,10,15,20\}$; $\lambda_2$ is chosen from $\{0.1,0.3,0.5,1\}$.
CD-GNN peaks at $\lambda_1=10$, showing that a moderate $\mathcal{L}_{cf}$ weight best enforces shortcut–causal embedding independence. 
This is because setting $\lambda_1$ too high leads to excessive counterfactual shortcuts, inadvertently reintroducing harmful bias and degrading model performance. 
These results highlight that a properly $\mathcal{L}_{cf}$ is crucial for blocking spillover and show that shortcut reliance (even counterfactually) undermines learning the true graph structure in heterophilic settings.
Moreover, CD-GNN performance peaks at $\lambda_2=0.1$. Here, $\lambda_2$ weights the HSIC loss to enforce node-level independence between causal and shortcut embeddings.
If $\lambda_2$ is too small (e.g., 0.03 or removing $\mathcal{L}_{HSIC}$ altogether as in the ablation study), the model under‐penalizes node‐level independence, leaving shortcut bias partly intact.
If $\lambda_2$ is too large (e.g., $\geq 0.3$), the HSIC penalty begins to dominate the total loss, over‐regularizing both representations and degrading the model's capacity to fit true causal signals.

\section{Related Work}
\label{sec:relate work}
Unlike heterogeneous graphs~\cite{yu2026drift, wang2022off}, which describe diversity in node or edge types, heterophily is a distinct graph property where connected nodes tend to have dissimilar attributes or labels.
Most heterophily-oriented GNNs fall into two categories: non-local neighbor extension and architecture refinement. Non-local neighbor extension expands the neighbor set to include non-local but similar nodes. This is commonly achieved via higher-order neighbor mixing and potential neighbor discovery. MixHop~\cite{abu2019mixhop} exemplifies higher-order mixing by aggregating multi-hop messages with hop-specific transformations. Potential neighbor discovery methods identify similar nodes: Geom-GCN~\cite{pei2020geom} defines neighbors based on geometric distances in a latent space, while U-GCN and SimP-GCN~\cite{jin2021universal,jin2021node} construct kNN-based neighbor sets using feature similarity. 
GNN-architecture refinement methods aim to better exploit information from discovered neighbors without redefining the neighbor set. Their core idea is to distinguish similar and dissimilar neighbors via adaptive message aggregation. FAGCN~\cite{bo2021beyond} introduces edge-level weights and high-pass graph filters to capture high-frequency signals of heterophilic graphs. Similarly, ACM~\cite{luan2022revisiting}, NLGCN~\cite{liu2021non}, GPR-GNN~\cite{chien2020adaptive}, and BernNet~\cite{he2021bernnet} leverage high-pass filters to model global high-frequency signals. 
CAGNN~\cite{chen2023exploiting} learns the neighbor effect for each node using a von Neumann entropy-based metric. It employs a shared mixer module to adaptively weight and incorporate neighbor information based on the leaned neighbor effect. 
GGCN~\cite{yan2022two} unifies heterophily and oversmoothing by profiling nodes via node-level heterophily and relative degree. It corrects edge weights using degree-based structural information and feature-based weights.
Similarly, DMP~\cite{yang2021diverse} addresses attribute heterophily by learning attribute-wise aggregation weights, while GDAMNs~\cite{chen2022graph} targets label heterophily using hard and soft label attention to derive adaptive edge weights. 
LatGRL~\cite{shen2025heterophily} addresses semantic heterophily by constructing latent homophilic and heterophilic graphs via similarity mining. A dual-frequency (low- and high-frequency) fusion further separates shared and distinctive (heterophilic) signals.
Overall, these methods reweight neighbors using node-level signals or learned attention.
In contrast, our approach uncovers broader structural patterns by identifying inductive subgraphs that govern message propagation in heterophilic graphs.

\section{Conclusion}
We empirically and theoretically showed that the degraded performance of standard message‐passing GNNs on heterophilic graphs can be traced back to spurious ``shortcut'' associations induced by recurring inductive subgraphs. 
By reframing the problem through the lens of causal inference, our analysis uncovered two unblocked paths—confounding and spillover—that link these subgraphs directly to predictions and thereby mislead homophilic GNN architectures. Building on these insights, we introduced Causal Disentangled GNN (CD-GNN), a novel framework that (i) disentangles spurious inductive subgraphs from the true causal subgraphs and (ii) explicitly blocks both confounding and spillover channels.
Empirically, CD-GNN outperforms state-of-the-art heterophily‐aware GNNs across multiple real‐world datasets, delivering substantial gains in node classification accuracy, especially on graphs with extreme heterophily ratios.

\begin{acks}
This work is supported by the Hong Kong Research Grants Council under the General Research Fund (project no. P0047343).
This work is also supported by Curtin University RTCM Trailblazer Project (project no. PRO-701177).
\end{acks}

\appendix
\section{Proof of Thm~\ref{thm:theorem31prime} and Thm~\ref{thm:theorem32prime}}
\label{subsec:assumptions_subgraph}

We follow~\cite{yan2022two} and use the same notation for the graph, node labels/features, and the renormalized propagation operator.
Further, for each node $v_i$ and layer $\ell$, we allow the existence of an \emph{inductive subgraph}
$S_i^{(\ell)}\subseteq G_i^{(\ell)}$ within the $\ell$-hop computation graph of $i$. $S_i^{(\ell)}$ denotes an inductive subgraph that is relevant to the computation at layer $\ell$. We consider one-hop neighbor set $\mathcal N_i$ for node $v_i$. Let $\mathcal N_{i,S}^{(\ell)}$ denote the set of neighbors selected/covered by the subgraph $S_i^{(\ell)}$ for the layer-$\ell$
computation, and $\mathcal N_{i,R}^{(\ell)}$ denotes the remaining neighbors.

Define the subgraph-induced partition at layer $\ell$,
$\mathcal N_i^{(l)}=\mathcal N_{i,S}^{(\ell)}\cup\mathcal N_{i,R}^{(\ell)}$ with $\mathcal N_{i,R}^{(\ell)}=\mathcal N_i\setminus \mathcal N_{i,S}^{(\ell)}$,
$d_{i,S}^{(\ell)}=|\mathcal N_{i,S}^{(\ell)}|$, $d_i^{(l)}\triangleq d_{i,S}^{(\ell)}+d_{i,R}^{(\ell)}$.
Hence, we define the \emph{group-wise} quantities:
\begin{align}
&\bar r_i^{\ell}
\triangleq
\mathbb E_{\mathbf{A}\mid d_i}\!\left(
\frac{1}{d_i}\sum_{j\in\mathcal N_i} r_{ij}\,\Big|\, d_i,\, S_i^{(\ell)}
\right),
\label{eq:rbar_group_R_paper}
\end{align}
\begin{equation}
d^{(\ell)}_{i}\bar r^\ell_{i}:=d^{(\ell)}_{i,S}\bar r^\ell_{i,S}+d^{(\ell)}_{i,R}\bar r^\ell_{i,R}. 
\label{eq:rbar_decomposition_paper}
\end{equation}
For $g\in\{S,R\}$, we define
\[
\hat h_{i,g}^{\ell}:=\Pr\!\big(j\in\hat{\mathcal N}_i(\ell)\mid j\in\mathcal N_{i,g}^{(\ell)},\,d_i,y_i,S_i^{(\ell)}\big),
\]
and define $\bar r_{i,g}^{\ell}$ as the group-wise version of $\bar r_i^\ell$:
\begin{equation}
    \bar r^\ell_{i,g}:=\mathbb E_A\!\left[\frac{1}{d^{(\ell)}_{i,g}}\sum_{j\in\mathcal N^{(\ell)}_{i,g}} r_{ij}\ \Big|\ d_i,S_i^{(\ell)}\right], \,\, \, d_{i,g}^{(\ell)}:=|\mathcal N_{i,g}^{(\ell)}|,
\end{equation}
\begin{equation}
    \beta_i^\ell:=\frac{d_{i,S}^{(\ell)}\bar r_{i,S}^{\ell}}{d_{i,S}^{(\ell)}\bar r_{i,S}^{\ell}+d_{i,R}^{(\ell)}\bar r_{i,R}^{\ell}}\in[0,1],\,
\hat h_{i,\mathrm{eff}}^\ell:=\beta_i^\ell \hat h_{i,S}^\ell+(1-\beta_i^\ell)\hat h_{i,R}^\ell.
\label{eq:C4_heff_simple}
\end{equation}

\begin{assumption}
\label{ass:group_exchangeability}
We assume \emph{group-wise exchangeability}: the multisets
$\{(y_j,f_j^{(\ell)},d_j): j\in\mathcal N_{i,S}^{(\ell)}\}$ and
$\{(y_j,f_j^{(\ell)},d_j): j\in\mathcal N_{i,R}^{(\ell)}\}$
are exchangeable within each group (no exchangeability across groups).
\end{assumption}

These definitions ensure that the movement factors in Thm~\ref{thm:theorem31prime} and Thm~\ref{thm:theorem32prime})
retain the same functional form as their baseline counterparts while explicitly accounting for the inductive subgraph.

\begin{proof}[Proof of Theorem 3.1]
    Fix a node $i$ and condition on $(d_i,y_i,S_i^{(0)})$. From the renormalized propagation~\cite{yan2022two} at $\ell=0$,
\begin{equation}
f_i^{(1)}=\frac{1}{d_i+1}\Big(f_i^{(0)}+\sum_{j\in\mathcal N_i} r_{ij}f_j^{(0)}\Big)
\label{eq:thm31_update}
\end{equation}

Using the subgraph-induced partition $\mathcal N_i=\mathcal N^{(0)}_{i,S}\cup \mathcal N^{(0)}_{i,R}$ and linearity of expectation, $\mathbb E\!\left[f_i^{(1)}\mid d_i,y_i,S_i^{(0)}\right]=$
\begin{align}
&\frac{1}{d_i+1}\Big(
\mathbb E[f_i^{(0)}\mid y_i]
+\!\!\sum_{j\in\mathcal N^{(0)}_{i,S}}\!\! r_{ij}\,\mathbb E\!\left[f_j^{(0)}\mid d_i,y_i,S_i^{(0)},j\in\mathcal N^{(0)}_{i,S}\right]  \nonumber\\
&\qquad \qquad\qquad+\!\!\sum_{j\in\mathcal N^{(0)}_{i,R}}\!\! r_{ij}\,\mathbb E\!\left[f_j^{(0)}\mid d_i,y_i,S_i^{(0)},j\in\mathcal N^{(0)}_{i,R}\right]
\Big)
\label{eq:thm31_split}
\end{align}

Under the mean model in~\cite{yan2022two}, for any fixed $y_i$ we have
$\mathbb E[f^{(0)}\mid y\neq y_i]=-\rho\,\mathbb E[f_i^{(0)}\mid y_i]$.
According to Assumption~\ref{ass:group_exchangeability}, by the law of total expectation,
\begin{align}
\mathbb E\!\left[f_j^{(0)}\mid d_i,y_i,S_i^{(0)},j\in\mathcal N^{(0)}_{i,S}\right]
&=\big((1+\rho)h^0_{i,S}-\rho\big)\,\mathbb E[f_i^{(0)}\mid y_i], \label{eq:thm31_Smean}\\
\mathbb E\!\left[f_j^{(0)}\mid d_i,y_i,S_i^{(0)},j\in\mathcal N^{(0)}_{i,R}\right]
&=\big((1+\rho)h^0_{i,R}-\rho\big)\,\mathbb E[f_i^{(0)}\mid y_i]. \label{eq:thm31_Rmean}
\end{align}

Substituting \eqref{eq:thm31_Smean}--\eqref{eq:thm31_Rmean} into \eqref{eq:thm31_split} yields $\mathbb E\!\left[f_i^{(1)}\mid d_i,y_i,S_i^{(0)}\right]=$
\begin{equation}
\frac{1+W^0_{i,S}\big((1+\rho)h^0_{i,S}-\rho\big)+W^0_{i,R}\big((1+\rho)h^0_{i,R}-\rho\big)}{d_i+1}\,
\mathbb E\!\left[f_i^{(0)}\mid y_i\right]
\label{eq:thm31_gain_two}
\end{equation}
where $W^0_{i,S}:=\sum_{j\in\mathcal N^{(0)}_{i,S}} r_{ij}$ and $W^0_{i,R}:=\sum_{j\in\mathcal N^{(0)}_{i,R}} r_{ij}$.
Define the one-layer dominance weight $\beta_i^0:=\frac{W^0_{i,S}}{W^0_{i,S}+W^0_{i,R}}$ and the effective homophily
$h^0_{i,\mathrm{eff}}:=\beta_i^0 h^0_{i,S}+(1-\beta_i^0)h^0_{i,R}$, and let $W_i^0:=W^0_{i,S}+W^0_{i,R}$.
Then \eqref{eq:thm31_gain_two} simplifies to Eq.~\eqref{eq:C4_gain_simple}, completing the proof. 
\end{proof}

\begin{proof}[Proof of Theorem 3.2]
Fix layer $\ell$ and condition on $(d_i,y_i,S_i^{(\ell)})$. The baseline deep-layer analysis (Thm~3.2 in~\cite{yan2022two})
introduces the event that a neighbor contributes \emph{positively} at layer $\ell$. The neighbor set is denoted by $\hat{\mathcal N}_i(\ell)$.
Applying the same conditional-splitting argument as in the baseline proof but separately within the $S$-group and the $R$-group, the expected neighbor contribution decomposes additively into two parts, each taking the same affine form as the baseline term.
Collecting the two group contributions and factoring by Eq.~\eqref{eq:rbar_decomposition_paper}
yields the compact gain form
\[
\mathbb E\!\left[f_i^{(\ell+1)}\mid d_i,y_i,S_i^{(\ell)}\right]
=
\frac{\big((1+\rho_i^\ell)\hat h^\ell_{i,\mathrm{eff}}-\rho_i^\ell\big)\,d_i\bar r^\ell_i+1}{d_i+1}\,
\xi_i^\ell\,\mathbb E\!\left[f_i^{(0)}\mid y_i\right]
\]
which is exactly Eq.~\eqref{eq:C5_gain_simple}. 
\end{proof}

\bibliographystyle{ACM-Reference-Format}
\balance
\bibliography{bibs}

\end{document}